\documentclass{article}

\usepackage[preprint,nonatbib]{neurips_2019}
\usepackage[numbers]{natbib}

\usepackage[utf8]{inputenc} 
\usepackage[T1]{fontenc}    
\usepackage{hyperref}       
\usepackage{url}            
\usepackage{booktabs}       
\usepackage{amsfonts}       
\usepackage{nicefrac}       
\usepackage{microtype}      

\usepackage{amsmath,amsthm,amssymb,bm}
\usepackage{subfigure}
\usepackage{algorithm}
\usepackage{algpseudocode}
\algnewcommand\To{\textbf{to}\;}
\usepackage[disable]{todonotes}
\usepackage{enumitem}
\usepackage{sty/notation}
\DeclareMathOperator*{\argmax}{argmax}
\DeclareMathOperator*{\argmin}{argmin}
\DeclareMathOperator{\ex}{\mathbb E}
\DeclareMathOperator{\pr}{\mathbb P}

\DeclareMathOperator{\R}{\mathbb R}
\DeclareMathOperator{\N}{\mathbb N}

\DeclareMathOperator{\kl}{kl}
\DeclareMathOperator{\ch}{ch}
\DeclareMathOperator{\GLR}{GLR}

\let\inf\undefined
\let\min\undefined
\let\max\undefined
\DeclareMathOperator*{\inf}{\vphantom{sup}inf}
\DeclareMathOperator*{\min}{\vphantom{sup}min}
\DeclareMathOperator*{\max}{\vphantom{sup}max}

\let\top\intercal

\DeclareBoldMathCommand{\vmu}{\mu}
\DeclareBoldMathCommand{\vlambda}{\lambda}
\DeclareBoldMathCommand{\vtheta}{\theta}
\DeclareBoldMathCommand{\vxi}{\xi}
\DeclareBoldMathCommand{\veta}{\eta}
\DeclareBoldMathCommand{\vsigma}{\sigma}
\DeclareBoldMathCommand{\vN}{N}
\DeclareBoldMathCommand{\w}{w}
\DeclareBoldMathCommand{\q}{q}
\DeclareBoldMathCommand{\p}{p}
\DeclareBoldMathCommand{\e}{e}
\DeclareBoldMathCommand{\a}{a}
\DeclareBoldMathCommand{\u}{u}

\newcommand{\ihat}{\hat \imath}

\newtheorem{lemma}{Lemma}
\newtheorem{theorem}{Theorem}

\theoremstyle{definition}
\newtheorem{definition}{Definition}
\newtheorem{assumption}{Assumption}

\title{%
  Non-Asymptotic Pure Exploration by Solving Games
  }

\author{%
  R\'emy Degenne \\
  Centrum Wiskunde \& Informatica \\
  Science Park 123, 1098 XG Amsterdam \\
  \texttt{remy.degenne@cwi.nl}
  \And
  Wouter M. Koolen \\
  Centrum Wiskunde \& Informatica \\
  Science Park 123, 1098 XG Amsterdam \\
  \texttt{wmkoolen@cwi.nl} \\
  \And
  Pierre M\'enard \\
  Inria Lille \\
  40 Avenue Halley, 59650 Villeneuve-d’Ascq \\
  \texttt{menardprr@gmail.com}
}

\begin{document}

\maketitle

\begin{abstract}
  Pure exploration (aka active testing) is the fundamental task of sequentially gathering information to answer a query about a stochastic environment. Good algorithms make few mistakes and take few samples.

  Lower bounds (for multi-armed bandit models with arms in an exponential family) reveal that the sample complexity is determined by the solution to an optimisation problem. The existing state of the art algorithms achieve asymptotic optimality by solving a plug-in estimate of that optimisation problem at each step.

  We interpret the optimisation problem as an unknown game, and propose sampling rules based on iterative strategies to estimate and converge to its saddle point.
  We apply no-regret learners to obtain the first finite confidence guarantees that are adapted to the exponential family and which apply to any pure exploration query and bandit structure. Moreover, our algorithms only use a best response oracle instead of fully solving the optimisation problem.
\end{abstract}

\section{Introduction}

We study fundamental trade-offs arising in sequential interactive learning. We adopt the framework of Pure Exploration, in which the learning system interacts with its environment by performing a sequence of experiments, with the goal of maximising information gain. We aim to design general, efficient systems that can answer a given query with few experiments yet few mistakes.

As usual, we model the environment by a multi-armed bandit model with exponential family arms, and work in the fixed confidence ($\delta$-PAC) setting. Information-theoretic lower bounds \citep{garivier2016optimal} show that a certain number of samples is unavoidable to reach a certain confidence. Moreover, algorithms are developed \citep{garivier2016optimal} that match these lower bounds asymptotically, in the small confidence $\delta \to 0$ regime.

Our contribution is a framework for obtaining efficient algorithms with \emph{non-asymptotic guarantees}. The main object of study is the ``Pure Exploration Game'' \citep{multiple.answers}, a two-player zero-sum game that is central to lower bounds as well as to the widely used GLRT-based stopping rules. We develop iterative methods that provably converge to saddle-point behaviour.
The game itself is not known to the learner, and has to be explored and estimated on the fly.
Our methods are based on pairs of low-regret algorithms, combined with optimism and tracking.
We prove sample complexity guarantees for several combinations of algorithms, and discuss their computational and statistical trade-offs.

The rest of the introduction provides more detail on pure exploration problems, the pure exploration game, the connection between them, and expands on our contribution. We also review related work.

%
Our model for the environment is a $K$-armed bandit, i.e.\ distributions $(\nu_1,\ldots,\nu_K)$ on $\mathbb R$. We assume throughout that these distributions come from a one-dimensional exponential family, and we denote by $d(\mu,\lambda)$ the relative entropy (Kullback-Leibler divergence) from the distribution with mean $\mu$ to that with mean $\lambda$.
A pure exploration problem is parameterised by a set $\mathcal M$ of $K$-armed bandit models (the possible environments), a finite set $\mathcal I$ of candidate answers and a correct-answer function $i^* : \mathcal M {\to} \mathcal I$. We focus on \emph{Best Arm Identification}, for which $i^*(\vmu) = \argmax_i \mu_i$ and the \emph{Minimum Threshold} problem, which is defined for any fixed threshold $\gamma$ by $i^*(\vmu) = \mathbf 1_{\set*{\min_i \mu_i < \gamma}}$.
The goal of the learner is to learn $i^*(\vmu)$ confidently and efficiently by means of sequentially sampling from the arms of $\vmu$, no matter which $\vmu \in \mathcal M$ it faces.
%
%
When an algorithm sequentially interacts with $\vmu$, we denote by $N_t^k$ and $\hat \mu_t^k$ the sample count and empirical mean estimate (these form a sufficient statistic) for arm $k$ after $t$ rounds. We write $\tau_\delta$ for the time at which the algorithm stops and $\ihat$ for the answer it recommends. The algorithm is correct (on a particular run) if it recommends $\ihat = i^*(\vmu)$ the correct answer for $\vmu$. An algorithm is $\delta$-PAC (or $\delta$-correct) if $\pr_\vmu(\ihat \neq i^*(\vmu)) \le \delta$ for each $\vmu \in \mathcal M$. Among $\delta$-PAC algorithms, we are interested in those minimising the sample complexity $\ex_\vmu[\tau_\delta]$. As it turns out, what can be achieved, and how, is captured by a certain game.

For each $\vmu \in \mathcal M$, \cite{multiple.answers} define the two-player zero-sum simultaneous-move \emph{Pure Exploration Game}: MAX plays an arm $k \in [K]$, MIN plays an ``alternative'' bandit model $\vlambda \in \mathcal M$ with a different correct answer $i^*(\vlambda) \neq i^*(\vmu)$. We denote the set of such alternatives to answer $i$ by $\neg i = \{\vlambda \in \mathcal M: i^*(\vlambda) \neq i\}$. MAX then receives payoff $d(\mu^k, \lambda^k)$ from MIN. As the payoff is neither concave in $k$ (since discrete) nor convex in $\vlambda$ (both domain and divergence are problematic), we will analyse the game by sequencing the moves and considering a mixed strategy for the player moving first. With MAX moving first and playing a mixed strategy $k \sim \w\in\triangle_K$ (we identify distributions over $[K]$ and the simplex $\triangle_K$), the value of the game is
\begin{equation}\label{eq:game}
  D_\vmu
  ~\df~
  \sup_{\w \in \triangle_K} D_\vmu(\w)
  \qquad
  \text{where}
  \qquad
  D_\vmu(\w)
  ~\df~
  \inf_{\vlambda \in \mathcal M
    :
    i^*(\vlambda) \neq i^*(\vmu)}
    \sum_{k=1}^K w^k d(\mu^k, \lambda^k)
  .
\end{equation}
We denote a minimiser of $D_\vmu$ by $\w^*(\vmu)$ and call it an \emph{oracle allocation}. The analogue where MIN plays first using a mixed strategy $\vlambda \sim \q\in\pr(\neg i^*(\vmu))$ (distributions over that set) is proposed and analysed in \cite{multiple.answers}. Despite the baroque domain of $\vlambda$ in \eqref{eq:game}, there always exist minimax $\q$ supported on $\le K$ points due to dimension constraints.

The Pure Exploration Game is essential to both characterising the complexity of learning, and also to algorithm design.
Namely, first, any $\delta$-correct algorithm has sample complexity for each bandit $\vmu \in \mathcal M$ at least
\(
  \ex_\vmu[\tau_\delta]
  \ge
  \wfrac{
    \kl(\delta, 1-\delta)
  }{
    D_\vmu
  }
  \approx
  \wfrac{
    \ln \frac{1}{\delta}
  }{
    D_\vmu
  }
  \)\todo{what is $\kl$?},
and matching this rate requires sampling proportions $\wfrac{\ex_\vmu[\vN_{\tau_\delta}]}{\ex_\vmu[\tau_\delta]}$ converging to $\w^*(\vmu)$ \cite[see][]{garivier2016optimal}. Moreover, second, the general approach for obtaining $\delta$-correct algorithms is based on the Generalised Likelihood Ratio Test (GLRT) statistic $Z_t \df t D_{\hat \vmu_t}\del*{\wfrac{\vN_t}{t}}$. There are universal thresholds $\beta(t, \delta) \approx \ln \frac{1}{\delta} + \frac{K}{2} \ln \ln \frac{t}{\delta}$ \cite[see e.g.][]{garivier2013informational, garivier2016optimal, mixmart, magureanu2014lipschitz} such that
\(
  \pr_\vmu \set*{
    \exists t : Z_t \ge \beta(\delta, t)
  }
  \le
  \delta
  \)
  for any $\vmu \in \mathcal M$.
  Hence stopping when $Z_t \ge \beta(t, \delta)$ and recommending $\ihat = i^*(\hat \vmu_t)$ is $\delta$-correct for any sampling rule. Maximising the GLRT to stop as early as possible is achieved by the sampling proportions $\vN_t/t = \w^*(\hat \vmu_t)$.

  These considerations show that any successful Pure Exploration agent needs to (approximately) solve the Pure Exploration Game $D_\vmu$.
The Track-and-Stop approach, pioneered by \cite{garivier2016optimal}, ensures that $\hat \vmu_t \to \vmu$ using \emph{forced exploration}, and $\vN_t/t \to \w^*(\hat \vmu_t)$ using \emph{tracking}. Continuity of $\w^*$ and $D_\vmu$ then yields that $Z_t \approx t D_\vmu(\w^*(\vmu)) = t D_\vmu$. The GLRT stopping rule triggers when $t = \wfrac{\beta(\delta,t)}{D_\vmu} \approx \wfrac{\ln \frac{1}{\delta}}{D_\vmu}$, meeting the lower bound in the asymptotic regime $\delta \to 0$.

\paragraph{Our contributions.}
We explore methods to solve the Pure Exploration game $D_\vmu$ associated with the unknown bandit model $\vmu$, and discusses their statistical and computational trade-offs. We look at solving the game iteratively, by instantiating a low-regret online learners for each player. In particular for the $k$-player we use a self-tuning instance of Exponentiated Gradient called AdaHedge \cite{ftl.jmlr}. The $\vlambda$-player  needs to play a distribution to deal with non-convexity; we consider Follow the Perturbed Leader as well as an ensemble of Online Gradient Descent experts. We show how a combination of optimistic gradient estimates, concentration of measure arguments and regret guarantees combine to deliver the first non-asymptotic sample complexity guarantees (which retain asymptotic optimality for $\delta \to 0$). The advantage of this approach is that it only requires a best response oracle (\ref{eq:game}, right) instead of a computationally more costly max-min oracle (\ref{eq:game}, left) employed by Track-and-Stop. Going the other extreme, we also develop Optimistic Track-and-Stop based on a max-max-min oracle (the outer max implementing optimism over a confidence region for $\vmu$), which trades increased computation for tighter sample complexity guarantees with simpler proofs.

Our cocktail sheds new light on the trade-offs involved in the design of pure exploration algorithms.
We show how ``big-hammer'' forced exploration can be refined using problem-adapted optimism.
We show how tracking is unnecessary when the $k$ player goes second. We show how computational complexity can be traded off using oracles of various sophistication. And finally, we validate our approach empirically in benchmark experiments at practical $\delta$, and find that our algorithms are either competitive with Track-and-Stop (dense $\w^*$) or dominate it (sparse $\w^*$).

\paragraph{Related work}

Besides maximising information gain, there is a vast literature on maximising reward in multi-armed bandit models for which a good starting point is \cite{bandit.book}.
The canonical Pure Exploration problem is Best Arm Identification \citep{DBLP:conf/colt/Even-DarMM02,Bubeckal11}, which is actively studied in the fixed confidence, fixed budget and simple regret settings \cite[Ch.~33]{bandit.book}. Its sample complexity as a function of the confidence level $\delta$ has been analysed very thoroughly in the (sub)-Gaussian case, where we have a rather complete picture, even including lower order terms \cite{pmlr-v65-chen17b}. \cite{on.the.complexity} initiated the quest for correct instance-dependent constants for arms from any exponential family. \cite{simchowitz2017simulator} stresses the importance of the ``moderate confidence'' regime $\delta \gg 0$. Although it is not the focus here, we do believe that it is crucial to obtain the right problem dependence not only in $\ln \frac{1}{\delta}$ but also in $K$ and other structural parameters, as the latter may in practice dominate the sample complexity.

Pure Exploration queries beyond Best Arm include Top-$M$ \cite{Shivaram:al10}, Thresholding \cite{thresholding}, Minimum Threshold \cite{kaufmann2018sequential}, Combinatorial Bandits \cite{Chen14ComBAI}, pure-strategy Nash equilibria \cite{pmlr-v70-zhou17b} and Monte-Carlo Tree Search \cite{FindTopWinner}. There is also significant interest in these problems in structured bandit models, including Rank-one \cite{rank.one.bandit}, Lipschitz \cite{magureanu2014lipschitz}, Monotonic \cite{garivier2017thresholding}, Unimodal \cite{combes2014unimodal} and Unit-Sum \cite{simchowitz2017simulator}.
Our framework applies to all these cases. Problems with multiple correct answers were recently considered by \cite{multiple.answers}. Existing learning strategies do not work unmodified; some fail and others need to be generalised.

Optimism is ubiquitous in bandit optimisation since \cite{Aueral02}, and was adapted to pure exploration by \cite{Shivaramal12}. We are not aware of optimism being used to solve unknown min-max problems. Optimism was employed in the UCB Frank-Wolfe method by \cite{berthet2017bandit} for maximising an unknown smooth function faster. We do not currently know how to make use of such fast rate results. For games the best response value is a non-smooth function of the action.

Using a pair of independent no-regret learners to solve a fixed and known game goes back to \cite{freund1999adaptive}. More recently game dynamics were used to explain (Nesterov) acceleration in offline optimisation \cite{DBLP:conf/nips/WangA18}. Ensuring faster convergence with coordinating learners is an active area of research \cite{DBLP:conf/nips/RakhlinS13}. Unfortunately, we currently do not know how to obtain an advantage in this way, as our main learning overhead comes from concentration, not regret.

\section{Algorithms with finite confidence sample complexity bounds}

We introduce a family of algorithms, presented as Algorithm~\ref{alg:the_algorithm}, with sample complexity bounds for non-asymptotic confidence. It uses the following ingredients: the GLRT stopping rule, a saddle point algorithm (possibly formed by two regret minimization algorithms) and optimistic loss estimates.

\subsection{Model and assumption: sub-Gaussian exponential families.}

We suppose that the distributions belong to a known one-parameter exponential family. That is, there is a reference measure $\nu_0$ and parameters $\eta_1,\ldots,\eta_K \in\R$ such that the distribution of arm $k\in[K]$ is defined by
$
\nicefrac{d\nu_k}{d\nu_0}(x) \propto e^{\eta_k x}
$. 
Examples include Gaussians with a given variance or Bernoulli with means in $(0,1)$.
All results can be extended to arms each in a possibly different known exponential family.
Let $\Theta$ be the open interval of possible means of such distributions.
A distribution $\nu$ is said to be $\sigma^2$-sub-Gaussian if for all $u\in\R$,
$
\log \ex_{X\sim\nu}e^{u(X-\ex_{X\sim\nu}[X])} \leq \frac{\sigma^2}{2} u^2 
$
.
An exponential family has all distributions sub-Gaussian with constant $\sigma^2$ iff for all $\mu, \lambda \in \Theta$, it verifies $d(\mu,\lambda) \geq \frac{1}{2\sigma^2}(\mu-\lambda)^2$.

\begin{assumption}\label{ass:sub-Gaussian}
The arm distributions belong to sub-Gaussian exponential families with constant $\sigma^2$.
\end{assumption}

\begin{assumption}\label{ass:compact_M}
There exists a closed interval $[\mu_{\min},\mu_{\max}] \subset \Theta$ such that $\mathcal{M}\subseteq [\mu_{\min},\mu_{\max}]^K$.
\end{assumption}

As a consequence of Assumption~\ref{ass:compact_M}, there exists $L,D>0$ such that for all $y\in[\mu_{\min},\mu_{\max}]$, the function $x\mapsto d(x,y)$ is $L$-Lipschitz on $[\mu_{\min},\mu_{\max}]$ and $d(x,y)\leq D$. 
Assumption~\ref{ass:sub-Gaussian} is implied by Assumption~\ref{ass:compact_M}. Both are discussed in Appendix~\ref{sec:ass_discussion}. In particular, Assumption~\ref{ass:compact_M} can often be relaxed. $L$ and $D$ will appear in the sample complexity bounds but none of our algorithms use them explicitly. 

Everywhere below, $\hat{\vmu}_t$ denotes the orthogonal projection of the empirical mean onto $[\mu_{\min},\mu_{\max}]^K$, with one possible exception: the GLRT stopping rule may use it either projected or not, indifferently.

\subsection{Algorithmic ingredients}\label{sec:alging}

\begin{algorithm}[t]
  \begin{algorithmic}[1]
    \Require Algorithms $\mathcal A^k$ and $\mathcal{A}^\vlambda$, stopping threshold $\beta(t,\delta)$ and exploration bonus $f(t)$.
\State Sample each arm once and form estimate $\hat{\vmu}_K$.
  \For{$t=K+1,\ldots$}
\State For $k\in[K]$,  let
$[\alpha_t^k, \beta_t^k] = \{\xi: N_{t-1}^k d(\hat{\mu}_{t-1}^k, \xi) \leq f(t{-}1)\}$.
\Comment{KL confidence intervals}
\State
Let $\tilde{\vmu}_{t-1} = \argmin_{\vlambda\in\mathcal{M}\cap \bigtimes_{k=1}^K[\alpha_t^k, \beta_t^k]} \sum_{k=1}^K N_{t-1}^k d(\hat{\mu}_{t-1}^k, \lambda^k)$.
\Comment{
  $ = \hat{\vmu}_{t-1}$ if $\hat{\vmu}_{t-1} \in \mathcal{M}$}
\State Let $i_t = i^*(\tilde{\vmu}_{t-1})$.
\State Stop and output $\ihat=i_t$ \textbf{if} $\inf_{\vlambda \in \neg i_t} \sum_k N_{t-1}^k d(\hat \mu_{t-1}^k, \lambda^k) > \beta(t,\delta)$. \Comment{GLRT Stopping rule}
\State Get $\w_t$ and $\q_t$ from $\mathcal{A}^k_{i_t}$ and $\mathcal{A}^\vlambda_{i_t}$.
\State For $k\in[K]$, let
$
U_t^k = \max \Big\{ f(t{-}1)/N_{t-1}^k, \max_{\xi\in\{\alpha_t^k,\beta_t^k\}} \ex_{\vlambda\sim \q_t}d\del{\xi, \lambda^k} \Big\}
$. \Comment{Optimism}
\State Feed $\mathcal A_{i_t}^k$ the loss
$
    \ell_t^\w(\w)
    =
    - \sum_{k=1}^K w^k U_t^k
    .
$
\State Feed $\mathcal{A}_{i_t}^\vlambda$ the loss $\ell_t^\vlambda(\q) = \ex_{\vlambda\sim \q}\sum_{k=1}^K w_t^k d(\hat{\mu}_{t-1}^k, \lambda^k)$ .
\State Pick arm $k_t = \argmin_k N_{t-1}^k / \sum_{s=1}^t w_s^k$. \Comment{Cumulative tracking}
\State Observe sample $X_t \sim \nu_{k_t}$. Update $\hat{\vmu}_t$.
  \EndFor
\caption{Pure exploration meta-algorithm.}
\label{alg:the_algorithm}
\end{algorithmic}
\end{algorithm}

\paragraph{Stopping and recommendation rules.}

The algorithm stops if any one of $|\mathcal{I}|$ GLRT tests succeeds \citep{garivier2016optimal}. Let $\mathcal{L}_\vmu$ denote the likelihood under the model parametrized by $\vmu$. The generalized log-likelihood ratio between a set $\Lambda$ and the whole parameter space $\Theta^K$ is
\begin{align*}
\mbox{GLR}_t^{\Theta^K}(\Lambda)
= \log \frac{\sup_{\tilde{\vmu} \in {\Theta^K}} \mathcal{L}_{\tilde{\vmu}}(X_1,\ldots,X_t)}{\sup_{\vlambda \in \Lambda} \mathcal{L}_{\vlambda}(X_1,\ldots,X_t)}
= \inf_{\vlambda\in\Lambda} \sum_{k\in[K]} N_t^k d(\hat{\mu}_t^k, \lambda^k) \: .
\end{align*}
By concentration of measure arguments, we may find $\beta(t,\delta)$ such that with probability greater than $1-\delta$, for all $t\in\N$, $\mbox{GLR}_t^{\Theta^K}(\{\vmu\}) \leq \beta(t, \delta)$ \cite[see ][]{garivier2013informational, garivier2016optimal, mixmart, magureanu2014lipschitz}. 
Test $i\in\mathcal{I}$ succeeds if $\GLR_t^{\Theta^K}(\neg i) > \beta(t,\delta)$. If the algorithm stops because of test $i$, recommend $\ihat=i$. If several tests succeed at the same time, choose arbitrarily among these.

\begin{theorem}\label{th:GLRT_delta_correct}
Any algorithm using the GLRT stopping and recommendation rules with threshold $\beta(t,\delta)$ such that $\pr_\vmu\{\mbox{GLR}_t^{\Theta^K}(\{\vmu\}) > \beta(t, \delta)\}\leq \delta$ is $\delta$-correct.
\end{theorem}

\paragraph{A game with two players}

An algorithm is unable to stop at time $t$ if the stopping condition is not met, i.e.
\begin{align*}
\beta(t, \delta)
&\geq \inf_{\vlambda \in \neg i^*(\hat{\vmu}_t)} \sum_{k\in[K]} N_t^k d(\hat{\mu}_t^k, \lambda^k) \: .
\end{align*}
In order to stop early, the right hand side has to be maximized, i.e.\ made
close to $t\sup_{\w\in\triangle_K}\inf_{\vlambda \in \neg i^*(\hat{\vmu}_t)} \sum_{k\in[K]} w_t^k d(\hat{\mu}_t^k, \lambda^k) = tD_{\hat{\vmu}_t}\approx tD_\vmu$.
Then with $\beta(t,\delta) \approx \log\nicefrac{1}{\delta} + o(t)$ we obtain $t\leq \log(1/\delta)/D_\vmu$ up to lower order terms, i.e.\ the stopping time is close to optimality.

We propose to approach that max-min saddle-point by implementing two iterative algorithms, $\mathcal{A}^k$ and $\mathcal{A}^\vlambda$, for the $k$-player and a $\vlambda$-player. Our sample complexity bound is a function of two quantities $R_t^k$
and $R_t^\vlambda$
, regret bounds of algorithms $\mathcal{A}^k$ and $\mathcal{A}^\vlambda$ when used for $t$ steps on appropriate losses.

One player of our choice goes first. The second player can see the action of the first, see the corresponding loss function and use an algorithm with zero regret (e.g.\ Best-Response or Be-The-Leader). One of the players has to play distributions on its action set. We have one of the following:
\begin{enumerate}[nolistsep]
\item $\vlambda$-player plays first and uses a distribution in $\pr(\neg i_t)$. The $k$-player plays $k_t\in[K]$.
\item $k$-player plays first and uses $\w_t\in\triangle_K$ (distribution over $[K]$). The $\vlambda$-player plays $\vlambda_t\in\neg i_t$.
\item Both players play distributions and go in any order, or concurrently. 
\end{enumerate}
Algorithm~\ref{alg:the_algorithm} presents two players playing concurrently but can be modified: if for example $\vlambda$ plays second, then it gets to see $\ell_t^\vlambda(\q)$ before computing $\q_t$.

The sampling rule at stage $t$ first computes the most likely answer $i_t$ for $\hat{\vmu}_{t-1}$. If the set over which the algorithm optimizes at line 4 is empty, $i_t$ is arbitrary.
The $k$-player plays $\w_t$ coming from $\mathcal{A}^k_{i_t}$, an instance of $\mathcal{A}^k$ running only on the rounds on which the selected answer is that $i_t$. The $\vlambda$-player similarly uses an instance $\mathcal{A}^\vlambda_{i_t}$ of $\mathcal{A}^\vlambda$.

\paragraph{Tracking.} Since a single arm has to be pulled, if the $k$-player plays $\w\in\triangle_K$ an additional procedure is needed to translate that play into a sampling rule. We use a so-called tracking procedure, $k_t = \argmin_{k\in[K]} N_{t-1}^k / \sum_{s=1}^t w_s^k$ , which ensures that  $\sum_{s=1}^t w_s^k - (K-1) \leq N_t^k \leq \sum_{s=1}^t w_s^k$ .

\paragraph{Optimism in face of uncertainty.}

Existing algorithms for general pure exploration use forced exploration to ensure convergence of $\hat{\vmu}_t$ to $\vmu$, making sure that every arm is sampled more than e.g.\ $\sqrt{t}$ times. 
We replace that method by the ``optimism in face of uncertainty'' principle, which gives a more adaptive exploration scheme.
While that heuristic is widely used in the bandit literature, this work is its first successful implementation for general pure exploration.
In Algorithm~\ref{alg:the_algorithm}, the $k$-player algorithm gets an optimistic loss depending on $\w_t$ and $\q_t$. The $\vlambda$-player gets a non-optimistic loss.

\subsection{Proof scheme and sample complexity result}\label{sec:proof_sketch}

In order to bound the sample complexity, we introduce a sequence of concentration events $\mathcal{E}_t = \{\forall s\leq t, \forall k \in [K], \: d(\hat{\mu}_s^k,\mu^k) \leq \frac{\widehat{W}((1+a)\log(t))}{N_s^k} \}$  for $a>0$ and $\widehat{W}(x) = x + \log x + \nicefrac{1}{2}$ . It verifies $\sum_{t=3}^{+\infty}\pr_\vmu(\mathcal{E}_t^c) \leq 2eK/a^2$ (see Appendix~\ref{sec:concentration} for a proof). The concentration intervals used in Algorihtm~\ref{alg:the_algorithm} are a function of $f(t) = \widehat{W}((1+a)(1+b)\log t)$ for $b>0$.

\begin{lemma}\label{lem:decomposition_of_sample_complexity}
Let $\mathcal{E}_t$ be an event and $T_0(\delta)\in\N$ be such that for $t\geq T_0(\delta)$, $\mathcal{E}_t \subseteq \{\tau_\delta \leq t\}$. Then
\begin{align*}
\ex_\vmu[\tau_\delta]
= \sum_{t=1}^{+\infty} \pr\{\tau_\delta > t\}
\leq T_0(\delta) + \sum_{t=T_0(\delta)}^{+\infty} \pr_\vmu (\mathcal{E}_t^c) \: .
\end{align*}
\end{lemma}

We now present briefly the steps of the proof for the stopping time upper bound before stating our main theorem on the sample complexity of Algorithm~\ref{alg:the_algorithm}. These steps are inexact and should be regarded as a guideline and not as rigorous computations. A full proof of our results can be found in the appendices (Appendix~\ref{sec:concentration} for concentration results, ~\ref{sec:tracking} for tracking and \ref{sec:sample_complexity_proof} for the main sample complexity proof).
We simplify the presentation by supposing that $i_t=i^*(\vmu)$ throughout (the main proof will show this may fail only $o(t)$ rounds). For $t< \tau_\delta$, under concentration event $\mathcal{E}_t$,
\begin{align*}
\beta(t,\delta)
&\geq \inf_{\vlambda \in \neg i^*(\vmu)} \sum_{k\in[K]} N_t^k d(\hat{\mu}_t^k,\lambda^k) & \text{(stopping condition)}\\
&\geq \inf_{\vlambda \in \neg i^*(\vmu)} \sum_{s\in[t]}\sum_{k\in[K]} w_s^k d(\hat{\mu}_t^k,\lambda^k) - KD & \text{(tracking)}\\
&\geq \inf_{\vlambda \in \neg i^*(\vmu)} \sum_{s\in[t]}\sum_{k\in[K]} w_s^k d(\hat{\mu}_{s-1}^k,\lambda^k) - \mathcal{O}(\sqrt{t\log(t)}) \ . & \text{(concentration)}
\end{align*}
The first term is now the infimum of a sum of losses, $\inf_{\vlambda \in \neg i^*(\vmu)} \sum_{s\in[t]}\ell_s^\vlambda(\vlambda)$. We use the regret property of the $\vlambda$-player's algorithm on those losses, then we introduce optimistic values $U_s^k$ such that for $\xi^k \in \{\mu^k,\hat{\mu}_{s-1}^k\}$ we have $\ex_{\vlambda\sim\q_s}d(\xi^k, \lambda^k) \leq U_s^k \leq \ex_{\vlambda\sim\q_s}d(\xi^k, \lambda^k) + \mathcal{O}(\sqrt{1/s})$.
\begin{align*}
\inf_{\vlambda \in \neg i^*(\vmu)} \sum_{s\in[t]}\sum_{k\in[K]} w_s^k d(\hat{\mu}_{s-1}^k,\lambda^k)
&\geq  \sum_{s\in[t]} \mathop{\mathbb{E}}_{\vlambda\sim\q_s}\sum_{k\in[K]}w_s^k d(\hat{\mu}_{s-1}^k,\lambda^k) - R_t^\vlambda & \text{(regret $\vlambda$)}\\
&\geq \sum_{s\in[t]} \sum_{k\in[K]}w_s^k U_s^k - \mathcal{O}(\sqrt{t}) - R_t^\vlambda  & \text{(optimism)} \\
&\geq \max_{k\in[K]} \sum_{s\in[t]} U_s^k - R_t^k - \mathcal{O}(\sqrt{t}) - R_t^\vlambda  & \text{(regret $\w$)}\\
&\geq \max_{k\in[K]} \sum_{s\in[t]} \mathop{\mathbb{E}}_{\vlambda\sim\q_s}\hspace{-3pt} d(\mu^k,\lambda^k) - R_t^k - \mathcal{O}(\sqrt{t}) - R_t^\vlambda & \text{(optimism)}
\end{align*}
Finally, $\nicefrac{1}{t}\sum_{s\in[t]}\ex_{\vlambda\sim\q_s}$ is itself the expectation of another distribution on $\pr(\neg i^*(\vmu))$. Hence
\begin{align*}
\max_{k\in[K]} \sum_{s\in[t]} \ex_{\vlambda\sim\q_s}d(\mu^k,\lambda^k) \geq t\inf_\q \max_k\ex_{\vlambda\sim\q}d(\mu^k,\lambda^k) = t D_\vmu \ .
\end{align*}
Putting these inequalities together, we get finally an inequality on such a $t<\tau_\delta$.
The exact result we obtain is the following Theorem, proved in Appendix~\ref{sec:sample_complexity_proof}.

\begin{theorem}\label{th:sample_complexity}
Under Assumption~\ref{ass:compact_M}, the sample complexity of Algorithm~\ref{alg:the_algorithm} on model $\vmu\in\mathcal{M}$ is
\begin{align*}
\ex_\vmu[\tau_\delta] &\leq T_0(\delta) + \frac{2eK}{a^2}
\quad\text{with}\quad
T_0(\delta) =  \max \{t\in\N:  t \leq \frac{\beta(t,\delta)}{D_\vmu} + C_\vmu(R_t^\vlambda + R_t^k + h(t))\}\: ,
\end{align*}
where $C_\vmu$ depends on $\vmu$ and $\mathcal{M}$ and $h(t) = \mathcal{O}(\sqrt{t\log(t)})$. See Appendix~\ref{sec:sample_complexity_proof} for an exact definition.
\end{theorem}

The forms of $h(t)$ and of $T_0(\delta)$ depend on the particular algorithm but we now show how an inequality of that type translates into $T_0(\delta)$. The next lemma is a consequence of the concavity of $t\mapsto\sqrt{t\log t}$.

\begin{lemma}
Suppose that $t\in\R$ verifies the equation $t - C\sqrt{t\log t} \leq \frac{\log 1/\delta}{D_\vmu}$. Then for $T^*_\delta = \frac{\log 1/\delta}{D_\vmu}$,
\begin{align*}
t \leq \frac{\log 1/\delta}{D_\vmu}  \bigg(  1 + C\sqrt{\frac{\log T^*_\delta}{T^*_\delta}} \frac{1}{1 - C\frac{1+\log T^*_\delta}{2\sqrt{T^*_\delta\log T^*_\delta}}} \bigg) \: .
\end{align*}
\end{lemma}

\section{Practical Implementations}
Next we discuss instantiating no-regret learners. We consider a hierarchy of computational oracles:
\begin{enumerate}[nolistsep]
\item Min aka Best-Response oracle: obtain for any $i\in\mathcal{I}$, $\w\in\triangle_K$ and $\vxi\in\Theta^K$ a minimizer in $\neg i$ of
$\vlambda\mapsto\sum_{k\in[K]} w^k d(\xi^k, \lambda^k)$ .
\item Max-min aka Game-Solving oracle: obtain for any $i\in\mathcal{I}$ and $\vxi\in\Theta^K$ a vector $\w^*\in \triangle_K $ such that there is a Nash equilibrium $(\w^*, \q^*) \in \triangle_K \times \pr(\neg i)$ for the zero-sum game with reward $d(\xi^k,\lambda^k)$ with the $k$-player using the mixed strategy $\w^*$.
\item Max-max-min oracle: for any confidence region $\mathcal{C}=[a_1,b_1]\times \ldots \times [a_K,b_K]$, obtain $(\vmu^+,i^+,\w^+)$ with $(\vmu^+, i^+) = \argmax_{\vxi\in\mathcal{C},i\in\mathcal{I}} \sup_{\w\in\triangle_K}\inf_{\vlambda\in\neg i} \sum_{k=1}^K w^k d(\xi^k, \lambda^k)$ and $\w^+$ a $k$-player strategy of a Nash equilibrium of the game with reward $d(\mu^{+ k}, \lambda^k)$.
\end{enumerate}

For Minimum Threshold all oracles can be evaluated in closed form in $O(K)$ time, and the same is true for Best Response in Best Arm Identification. Max-min for Best Arm requires binary search \cite{garivier2016optimal} and Max-max-min requires $O(K)$ max-min calls. See \citep{Menard19} for run-time data on Track-and-Stop (max-min oracle) and gradient ascent (min oracle) for Best Arm.
Our approach also extends naturally to min-max and max-min-max oracles, which we plan to incorporate in full detail in our future work.

\subsection{A Learning Algorithm for the $k$-Player vs Best-Response for the $\vlambda$-Player}\label{sec:DaBomb}

In this section the $k$-player plays first, employing a regret minimization algorithm for linear losses on the simplex to produce $\w_t\in\triangle_K$ at time $t$. We pick AdaHedge of \cite{ftl.jmlr}, which runs in $O(K)$ per round and adapts to the scale of the losses. The $\vlambda$-player goes second and can use a zero-regret algorithm: Best-Response. It plays
$
\q_t \mbox{ , a Dirac at } \vlambda_t \in \argmin_{\vlambda \in \neg i_t} \sum_{k\in[K]} w_t^k d(\hat{\mu}_{t-1}^k, \lambda^k) \: .
$

\begin{lemma}
  AdaHedge has regret $R_t^k \le \sqrt{\sum_{s \le t} b_s^2 \ln K} + \max_{s\le t} b_s \del{\frac{4}{3}\ln K + 2}$ where $b_s = \max_k U_s^k - \min_k U_s^k \le \max\set{D, f(s)}$ is the loss scale in round $s$, so that $R_t^k = \mathcal O(\sqrt{t \ln K} \ln t)$. Best-Response has no regret, $R_t^\vlambda \le 0$. The sample complexity is bounded per Theorem~\ref{th:sample_complexity}.
\end{lemma}
We expect that in practice the scale converges to $b_s \to D_\vmu$ after a transitory startup phase.

\textbf{Computational complexity:} one best-response oracle call per time step.

\subsection{Learning Algorithms for the $\vlambda$-Player vs Best Response for the $k$-Player}

Using a learner for the $\vlambda$-player removes the need for a tracking procedure. In this section the $k$-player goes second and uses Best-Response, with zero regret, i.e.\  $k_t = \argmax_{k\in[K]} U_t^k$ (see Algorithm~\ref{alg:the_algorithm}). After playing $\q_t\in \pr(\neg i_t)$, the $\vlambda$-player suffers loss $\ex_{\vlambda \sim \q_t} d(\hat\mu_{t-1}^{k_t}, \lambda^{k_t})$.

Most existing regret minimization algorithms do not apply since the function $\lambda \mapsto d(\mu, \lambda)$ is not convex in general and the action set $\neg i_t$ is also not convex. The challenge is to come up with an algorithm able to play distributions with only access to a best-response oracle.

\paragraph{Follow-The-Perturbed-Leader.}

Follow-The-Perturbed-Leader can sample points from a distribution on $\pr(\neg i)$ by only using best-response oracle calls on $\neg i$. The version we use here incorporates all the information available to the $\vlambda$-player: the loss of $\vlambda\in\neg i_t$ will be $d(\hat{\mu}_{t-1}^{k_t}, \lambda^{k_t})$ where the only unknown quantity is $k_t$. Let $\sigma_t\in \R^K$ be a random vector with independent exponentially distributed coordinates. The idea is that the distribution $\q_t$ played by the $\vlambda$-player should be the distribution of
\begin{align*}
\argmin_{\vlambda \in\neg i_t} \sum_{s=1}^{t-1} d(\hat{\mu}_{s-1}^{k_s},\lambda^{k_s}) + \sum_{k=1}^K \sigma_t^k d(\hat{\mu}_{t-1}^{k}, \lambda^{k}) \: .
\end{align*}
We show in Appendix~\ref{sec:ftpl_proof} that this argmin can be computed by a single best-response oracle call.
However, the $k$-player has to be able to compute the best response to $\q_t$. Since we cannot get the above distribution exactly, we instead take for $\q_t$ an empirical distribution from $t$ samples. A regret bound $R_t^\vlambda = \mathcal{O}(\sqrt{t\log t})$ for that algorithm is in Appendix~\ref{sec:ftpl_proof}. The sample complexity is then bounded by Theorem~\ref{th:sample_complexity}.

\textbf{Computational complexity:} $t$ best-response oracle calls at time step $t$.

\paragraph{Online Gradient Descent.}

While the learning problem for $\vlambda$ is hard in general, in several common cases the sets $\neg i$ have a simple structure. If these sets are unions of a finite number $J$ of convex sets and $\lambda\mapsto d(\mu,\lambda)$ is convex (i.e.\ for Gaussian or Bernoulli arm distributions), then we can use off-the-shelf regret algorithms. One gradient descent learner can be used on each convex set, and these $J$ experts are then aggregated by an exponential weights algorithm. This procedure would have $\mathcal{O}(\sqrt{t})$ regret.
The computational complexity is $J$ (convex) best-response oracle calls per time step.

\subsection{Optimistic Track-and-Stop.}\label{sec:OTaS}

At stage $t$, this algorithm computes $(\vmu^+,i_t) = \argmax_{\vxi,i} \sup_{\w\in\triangle_K}\inf_{\vlambda\in\neg i} \sum_{k=1}^K w^k d(\xi^k, \lambda^k)$ where $\vxi$ ranges over all points in $\Theta^K$ in a confidence region around $\hat{\vmu}_{t-1}$ and $i\in\mathcal{I}$. Then, the $k$-player plays $\w_t$ such that there exists a Nash equilibrium $(\w_t,\q_t)$ of the game with reward $d({\mu^+}^k, \lambda^k)$.
The proof of its sample complexity bound proceeds slightly differently from the sketch of part~\ref{sec:proof_sketch}, although the ingredients are still the GLRT, concentration, optimism and game-solving. The proof of the following lemma can be found in appendix~\ref{sec:ftpl_proof}.

\begin{lemma}
Take $b=1$ in the definition of $f(t)$. Let $h(t) = 2\sqrt{t}D_\vmu + 3L\sqrt{2\sigma^2 f(t)}(K^2 + (2\sqrt{2}+\nicefrac{1}{3})\sqrt{Kt}) + f(t)(K^2 + 2K\log(t/K)) + KD$. Then the expected sample complexity is at most $T_0(\delta) + \frac{2eK}{a^2}$, where $T_0(\delta)$ is the maximal $t\in\N$ such that
$t \leq (\beta(t,\delta) + h(t))/D_\vmu$ .
\end{lemma}

Note: the $K^2$ factors are due to the tracking. We conjecture that they should be $K\log K$ instead.

\textbf{Computational complexity:} one max-max-min oracle call per time step.

This algorithm is the most computationally expensive but has the best sample complexity upper bound, has a simpler proof and works well in experiments where computing the max-max-min oracle is feasible, like the Best Arm and Minimum Threshold problems (see section~\ref{sec:experiments}).

\section{Experiments}\label{sec:experiments}

The goal of our experiments is to empirically validate Algorithm~\ref{alg:the_algorithm} on benchmark problems for practical $\delta$. We use stylised stopping threshold $\beta(\delta, t) = \ln \frac{1+\ln t}{\delta}$ and exploration bonus $f(t) = \ln t$. Both are unlicensed by theory yet conservative in practise (the error frequency is way below $\delta$).
We use the following letter coding to designate sampling rules: \textbf{D} for AdaHedge vs Best-Response as advocated in Section~\ref{sec:DaBomb}, \textbf{T} for Track-and-Stop of \cite{garivier2016optimal}, \textbf{M} for the Gradient Ascent algorithm of \cite{Menard19}, \textbf{O} for Optimistic Track-and-Stop from Section~\ref{sec:OTaS}, \textbf{RR} for uniform, and \textbf{opt} for following the oracle proportions $\w^*(\vmu)$. We also ran all our experiments on a simplification of \textbf{D} that uses a single learner instead of partitioning the rounds according to $i_t$. We omit it from the results, as it was always within a few percent of \textbf{D}. We append \textbf{-C} or \textbf{-D} to indicate whether cumulative ($\vN_t \rightsquigarrow \sum_{s\le t} \w_s$) or direct ($\vN_t \rightsquigarrow t \w_t$) tracking \cite{garivier2016optimal} is employed. We finally note that we tune the learning rate of \textbf{M} in terms of (the unknown) $D_\vmu$.

We perform two series of experiments, one on Best Arm instances from \cite{garivier2016optimal, Menard19}, and one on Minimum Threshold instances from \cite{kaufmann2018sequential}. Two selected experiments are shown in Figure~\ref{fig:selected}, the others are included in Appendix~\ref{appx:experiments}. We contrast the empirical sample complexity with the lower bound $\kl(\delta, 1-\delta)/D_\vmu$, and with a more ``practical'' version, which indicates the time $t$ for which $t = \beta(t, \delta)/D_\vmu$, which is, approximately, the first time at which the GLRT stopping rule crosses the threshold $\beta$.

We see in Figures~\ref{fig:bai} and~\ref{fig:mt} that direct tracking \textbf{-D} has the advantage over cumulative tracking \textbf{-C} across the board, and that uniform sampling \textbf{RR} is sub-optimal as expected. In Figure~\ref{fig:bai} we see that \textbf{T} performs best, closely followed by \textbf{M} and \textbf{O}. Sampling from the oracle weights \textbf{opt} performs surprisingly  poorly (as also observed in \cite[][Table~1]{simchowitz2017simulator}). The main message of Figure~\ref{fig:mt} is that \textbf{T} can be highly sub-optimal. We comment on the reason in Appendix~\ref{appx:exper.mt}.
Asymptotic optimality of \textbf{T} implies that this effect disappears as $\delta \to 0$. However, for this example this kicks in excruciatingly slowly. Figure~\ref{fig:tas.bad} shows that \textbf{T} is still not competitive at $\delta=10^{-20}$. On the other hand, \textbf{O} performs best, closely followed by \textbf{M} and then \textbf{D}. Practically, we recommend using \textbf{O} if its computational cost is acceptable, \textbf{M} if an estimate of the problem scale is available for tuning, and \textbf{D} otherwise.

\begin{figure}
  \centering
  \subfigure[Best Arm for Bernoulli bandit model $\vmu = (0.3, 0.21, 0.2, 0.19, 0.18)$. The oracle weights are $\w^* = (0.34, 0.25, 0.18, 0.13, 0.10)$.\label{fig:bai}]{
    \includegraphics[width=.45\textwidth]{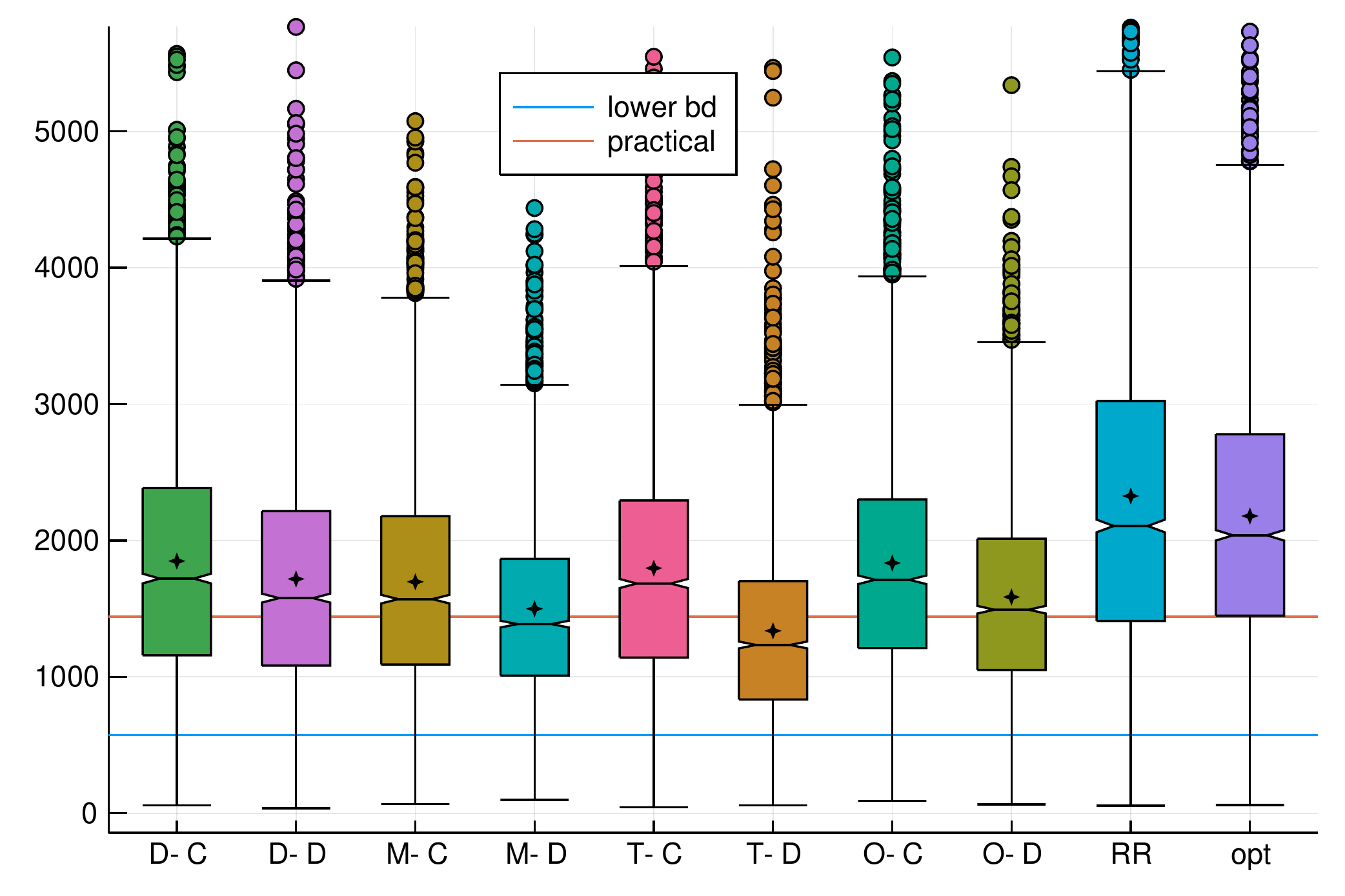}
  }%
  \hfill
  \subfigure[Minimum Threshold for Gaussian bandit model $\vmu = (0.5, 0.6)$ with threshold $\gamma=0.6$, $\w^* = \e_1$. Note the excessive sample complexity of T-C/ T-D.\label{fig:mt}]{
    \includegraphics[width=.45\textwidth]{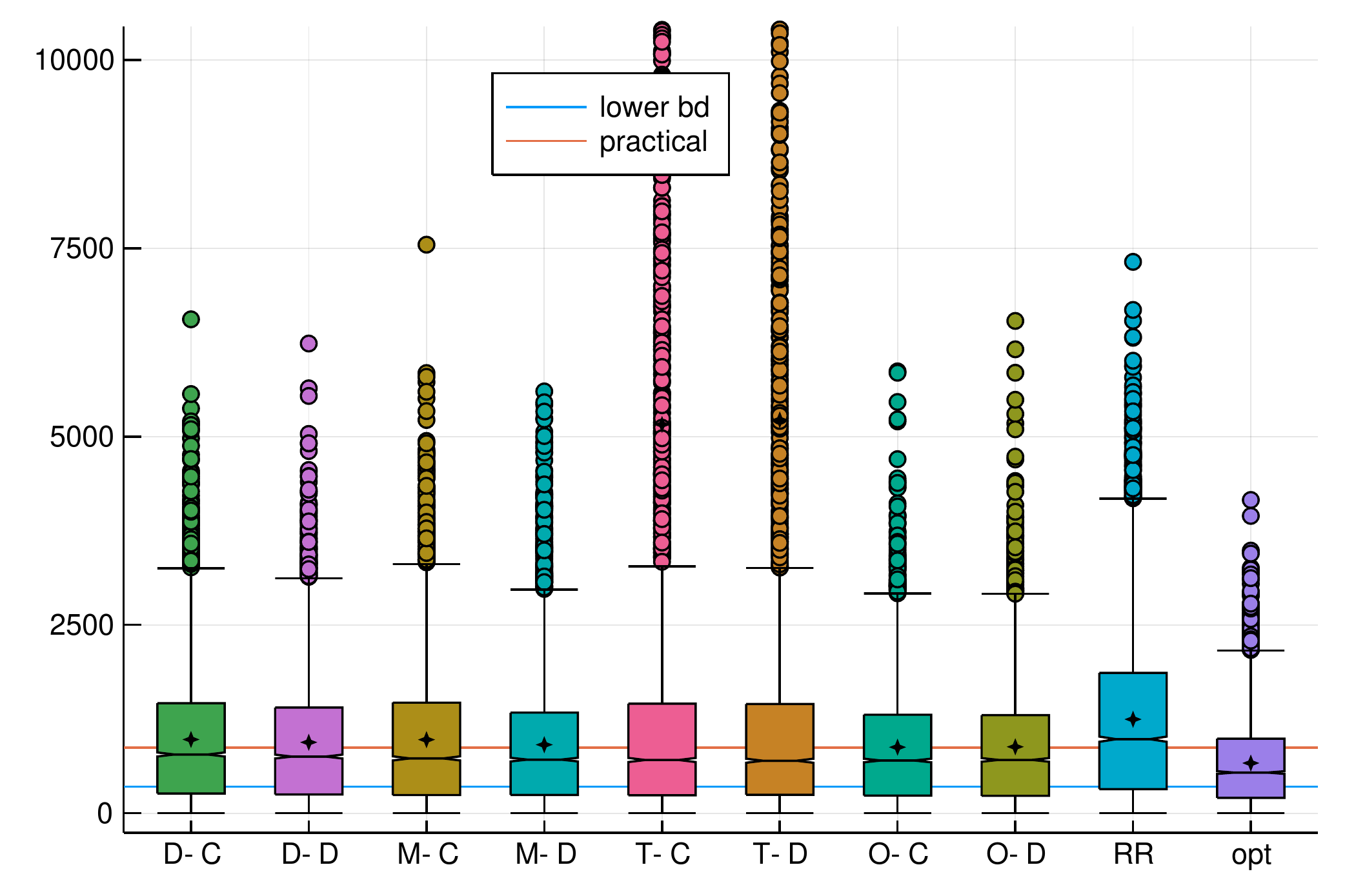}%
  }
  \caption{Selected experiments. In both cases $\delta=0.1$. Plots based on $3000$ and $5000$ runs.
  }\label{fig:selected}
\end{figure}

The gap between \textbf{opt} and \textbf{T} (or \textbf{O}) shows that Track-and-Stop outperforms its design motivation. It is an exciting open problem to understand exactly why, and to optimise for stopping early ($\vN_t/t \approx \w^*(\hat \vmu_t)$) while ensuring optimality ($\wfrac{\ex_\vmu[\vN_\tau]}{\ex_\vmu[\tau]} \approx \w^*(\vmu))$.

\section{Conclusion}

We leveraged the game point of view of the pure exploration problem, together with the use of the optimism principle, to derive algorithms with sample complexity guarantees for non-asymptotic confidence. Varying the flavours of optimism and saddle-point strategies leads to procedures with diverse tradeoffs between sample and computational complexities.
Our sample complexity bounds attain asymptotic optimality while offering guarantees for moderate confidence and the obtained algorithms are empirically sound. Our bounds however most probably do not depend optimally on the problem parameters, like the number of arms $K$. For BAI and the Top-K arms problems, lower bounds with lower order terms as well as matching algorithms were derived by \citep{simchowitz2017simulator}. A generalization of such lower bounds to the general pure exploration problem could shed light upon the optimal complexity across the full confidence spectrum.

The richness of existing saddle-point iterative algorithms may bring improved performance over our relatively simple choices. A smart algorithm could possibly take advantage of the stochastic nature of the losses instead of treating them as completely adversarial.


\subsubsection*{Acknowledgements}
We are grateful to Zakaria Mhammedi and Emilie Kaufmann for multiple generous discussions. Travel funding was provided by INRIA Associate Team ${}^6$PAC. The experiments were carried out on the Dutch national e-infrastructure with the support of SURF Cooperative.

\bibliographystyle{plain}
\bibliography{bib}

\newpage

\appendix

\section{Likelihood Ratio and Exponential Families}\label{sec:GLR}
\subsection{Canonical one-parameter exponential families}

We suppose that all arms have distributions in a canonical one-parameter exponential family. That is, there is a reference measure $\nu_0$ and parameters $\eta_1,\ldots,\eta_K \in\R$ such that the distribution of arm $k\in[K]$ is defined by
\begin{align*}
\nicefrac{d\nu_k}{d\nu_0}(x) \propto e^{\eta_k x - \psi(\eta_k)} \qquad \mbox{with} \qquad \psi(\eta) = \log\ex_{X\sim\nu_0} e^{\eta x} \: .
\end{align*}

Let $\phi$ be the convex conjugate of $\psi$, i.e. $\phi(x) = \sup_{y\in \mbox{dom }\psi} (xy - \psi(y))$. Let $\Theta\subset \R$ be the open interval on which the first derivative $\phi'$ is defined. The Kullback-Leibler divergence between the distributions of the exponential family with means $\mu$ and $\lambda$ in $\Theta$ is
\[
d(\mu, \lambda) = \phi(\mu) - \phi(\lambda) - (\mu-\lambda)\phi'(\lambda) \: .
\]

A distribution $\nu$ is said to be $\sigma^2$-sub-Gaussian if for all $u\in\R$,
$
\log \ex_{X\sim\nu}e^{u(X-\ex_{X\sim\nu}[X])} \leq \frac{\sigma^2}{2} u^2 
$
.
A canonical one-parameter exponential family has all distributions sub-Gaussian with constant $\sigma^2$ iff for all $\mu, \lambda \in \Theta$, it verifies $d(\mu,\lambda) \geq \frac{1}{2\sigma^2}(\mu-\lambda)^2$.

\subsection{The Generalized log-likelihood ratio}

The generalized log-likelihood ratio between the whole model space $\mathcal{M}$ and a subset $\Lambda\subseteq \mathcal{M}$ is
\begin{align*}
\mbox{GLR}_t^\mathcal{M}(\Lambda) = \log \frac{\sup_{\tilde{\vmu} \in \mathcal{M}} \mathcal{L}_{\tilde{\vmu}}(X_1,\ldots,X_t)}{\sup_{\vlambda \in \Lambda} \mathcal{L}_{\vlambda}(X_1,\ldots,X_t)} \: .
\end{align*}

In the case of a canonical one-parameter exponential family, the likelihood of the model with means $\vmu$ is
\begin{align*}
\mathcal{L}_\vmu(X_1,\ldots,X_t) = \prod_{s=1}^t \exp(\phi'(\mu^{k_s})(X_s - \mu^{k_s}) + \phi(\mu^{k_s})) d\nu_0(X_s)
\end{align*}
For $\vxi, \vlambda \in \mathcal{M}$ two mean vectors,
\begin{align*}
\log \frac{\mathcal{L}_\vxi(X_1,\ldots,X_t)}{\mathcal{L}_\vlambda(X_1,\ldots,X_t)}
= \sum_{s=1}^t d(X_s, \lambda^{k_s}) - d(X_s, \xi^{k_s})
= \sum_{k=1}^K N_t^k [d(\hat{\mu}_t^k, \lambda^k) - d(\hat{\mu}_t^k, \xi^k)] \: .
\end{align*}
The maximum likelihood estimator $\tilde{\vmu}_t$ corresponding to the data $X_1,\ldots,X_t$ is
\begin{align*}
\tilde{\vmu}_t = \argmin_{\vlambda \in\mathcal{M}} \sum_{k=1}^K N_t^k d(\hat{\mu}_t^k, \lambda^k) \: .
\end{align*}
The GLR for set $\Lambda$ is
\begin{align*}
\mbox{GLR}_t^\mathcal{M}(\Lambda)
&= \argmin_{\vlambda \in\Lambda} \sum_{k=1}^K N_t^k d(\hat{\mu}_t^k, \lambda^k) - \argmin_{\vlambda \in\mathcal{M}} \sum_{k=1}^K N_t^k d(\hat{\mu}_t^k, \lambda^k)\\
&= \argmin_{\vlambda \in\Lambda} \sum_{k=1}^K N_t^k d(\hat{\mu}_t^k, \lambda^k) - \sum_{k=1}^K N_t^k d(\hat{\mu}_t^k, \tilde{\mu}_t^k) \: .
\end{align*}

\section{Concentration Lemmas}\label{sec:concentration}

\subsection{Concentration bounds}

For $x>0$, let $\widehat{W}(x) = x + \log(x)$. Let $W_{-1}$ be the negative branch of the Lambert W function and for $x\geq 1$, $\overline{W}(x) = -W_{-1}(-e^{-x})$. Then
\begin{itemize}
\item For $x,y\geq 1$, $x-\log x \geq y \Leftrightarrow x \geq \overline{W}(y)$ .
\item For $x>1$, $\widehat{W}(x) \leq \overline{W}(x) \leq \widehat{W}(x) + \min\{\frac{1}{2}, \frac{1}{\sqrt{x}}\}$ .
\end{itemize}

\begin{lemma}[\citep{garivier2013informational}]
Let $Y_1^k,\ldots, Y_t^k$ be i.i.d. random variables in a canonical one-parameter exponential family with mean $\mu^k$. Then for $\alpha>0$,
\begin{align*}
\pr_\vmu\left\{ \exists s \leq t, s d(\frac{1}{s}\sum_{r=1}^s Y^k_r, \mu^k) \geq \alpha \right\} \leq 2e\log(t)e^{-(\alpha - \log \alpha)} \: .
\end{align*}
\end{lemma}

Remark that for $s\leq t$, the number of pulls verifies $N_s^k\leq t$. For $t>e$ and $\alpha = \overline{W}((1+a)\log t)$  with $a>0$, the lemma above implies
\begin{align*}
\pr_\vmu\left\{ \exists s\leq t,\: N_s^k d(\hat{\mu}_s^k, \mu^k) \geq \overline{W}((1+a) \log t) \right\} \leq 2e\frac{\log t}{t^{1+a}} \: .
\end{align*}

\begin{definition}\label{def:f}
Let $f(s) = \overline{W}((1+a)(1+b) \log s)$.
\end{definition}
For $s\geq t^{1/(1+b)}$, $\overline{W}((1+a)(1+b)\log s)  \geq \overline{W}((1+a)\log t)$, and when the event above happens,
\begin{align*}
d(\hat{\mu}_s^k, \mu^k) \leq \frac{f(s)}{N_s^k} \: . 
\end{align*}

\subsection{Main concentration event}

Concentration event for $t\geq 3$:
\begin{align*}
\mathcal{E}_t
&=
\left\{
	\forall s \leq t, \forall k\in[K]\:
	N_s^k d(\hat{\mu}_s^k, \mu^k) \leq \overline{W}((1+a) \log t)
\right\}
\end{align*}

\begin{lemma}\label{lem:probability_of_E_t^c}
\begin{align*}
\forall t\geq 3,\: \pr_\vmu(\mathcal{E}_t^c) \leq 2eK\frac{\log t}{t^{1+a}} \: ,
&& \sum_{t=3}^{+\infty}\pr_\vmu(\mathcal{E}_t^c) \leq \frac{2eK}{a^2} \: .
\end{align*}
\end{lemma}

\begin{proof}
\begin{align*}
\sum_{t=3}^{+\infty}\pr_\vmu(\mathcal{E}_t^c)
\leq 2eK \sum_{t=3}^{+\infty} \frac{\log t}{t^{1+a}} 
\leq 2eK \int_{x=1}^{+\infty} \frac{\log x}{x^{1+a}}dx
=    \frac{2eK}{a^2}\: .
\end{align*}
\end{proof}

\section{Tracking}\label{sec:tracking}

\begin{lemma}\label{lem:tracking}
Let $(\w_s)_{s\in\N} \in \triangle_K^{\N}$ be vectors in the simplex with $\w_1,\ldots,\w_K$ equal to the basis vectors. We recursively define for $t\in\N$,
\begin{align*}
& & \forall k\in[K],\: N_K^k &= 1 \: , \\
\forall t\geq K+1, \: k_t   &= \argmin_k \frac{N_{t-1}^k}{\sum_{s=1}^t w_s^k} \: ,
& \forall k\in[K],\: N_t^k &= \sum_{s=1}^{t} \mathbb{I}\{k_s=k\} \: .
\end{align*}
The tie-breaking for the $\argmin$ is arbitrary. Then for all $t\geq K$, all $k\in[K]$, 
\begin{align*}
\sum_{s=1}^t w_s^k - (K-1) \leq N_t^k \leq \sum_{s=1}^t w_s^k + 1 \: .
\end{align*}
\end{lemma}

\begin{proof}
Let $\Sigma_t^k = \sum_{s=1}^t w_s^k$. We start by proving the inequality on the right by induction. At $t=K$, for all $k$, $N_K^k = \Sigma_K^k = 1$.

Suppose now that $N_s^i\leq \Sigma_s^i+1$ for all $i\in[K]$ and all $s\leq t-1$. We prove that it also holds for $t$.

If $i\neq k_t$, by the induction hypothesis, $N_{t-1}^i \leq \Sigma_{t-1}^i + 1$. We obtain $N_t^i = N_{t-1}^i \leq \Sigma_{t-1}^i + 1 \leq \Sigma_t^i +1$.

If $i=k_t$, we use that $\sum_{j=1}^K N_{t-1}^j = t-1$ and $\sum_{j=1}^K \Sigma_t^j = t$ to say that $\min_j \frac{N_{t-1}^j}{\Sigma_t^j} \leq \frac{t-1}{t}\leq 1$. Since $k_t$ realizes that minimum, we have
\begin{align*}
\frac{N_t^{k_t}}{\Sigma_t^{k_t}} &= \frac{N_{t-1}^{k_t}}{\Sigma_t^{k_t}} + \frac{1}{\Sigma_t^{k_t}} \leq 1 + \frac{1}{\Sigma_t^{k_t}} \: .
\end{align*}
The inequality is proved for all $k\in[K]$ at $t$.

The lower bound for $N_t^i$ follows from the fact that $\sum_{i=1}^K N_t^i = \sum_{i=1}^K \Sigma_t^i = t$.

\begin{align*}
N_t^i = t - \sum_{j\neq i} N_t^j \geq t - \sum_{j\neq i} (\Sigma_t^j + 1) = \Sigma_t^i - (K-1) \: .
\end{align*}
\end{proof}

\begin{lemma}\label{lem:sum_w/sqrt(N)}
For $t \geq t_0 \geq 1$ and $(x_s)_{s\in[t]}$ non-negative real numbers such that $\sum_{s=1}^{t_0-1}x_s>0$,
\begin{align*}
\sum_{s=t_0}^t \frac{x_s}{\sqrt{\sum_{r=1}^s x_r}}
&\leq 2\sqrt{\sum_{s=1}^t x_s} - 2\sqrt{\sum_{s=1}^{t_0-1} x_s} \: .\\
\sum_{s=t_0}^t \frac{x_s}{\sum_{r=1}^s x_r}
&\leq \log(\sum_{s=1}^t x_s) - \log(\sum_{s=1}^{t_0-1} x_s) \: .
\end{align*}
\end{lemma}
\begin{proof}
By concavity of $x\mapsto\sqrt{x}$, we have $\sqrt{x} \leq \sqrt{x+y} - \frac{y}{2\sqrt{x+y}}$. We obtain $\frac{x_s}{\sqrt{\sum_{r=1}^s x_r}} \leq 2(\sqrt{\sum_{r=1}^s x_r} - \sqrt{\sum_{r=1}^{s-1} x_r})$ . The sum is then telescopic. The second result uses the concavity of $x\mapsto\log(x)$.
\end{proof}

\begin{lemma}\label{lem:concentration_term_w/sqrt(N)}
Let $(\w_s)_{s\in\N} \in \triangle_K^{\N}$ be vectors in the simplex. Let $N_t$ be defined as in Lemma~\ref{lem:tracking}. Then
\begin{align*}
\sum_{k=1}^K\sum_{s=K}^t \frac{w_s^k}{\sqrt{N_s^k}} \leq K^2 + 2\sqrt{Kt}
\quad \mbox{and} \quad
\sum_{k=1}^K\sum_{s=K+1}^t \frac{w_s^k}{\sqrt{N_{s-1}^k}} \leq K^2 + 2\sqrt{2Kt}
\: .
\end{align*}
\end{lemma}
\begin{proof}
We first prove the inequality on the left.
Let $t_0^k$ be the first time such that $\sum_{r=1}^{t_0^k-1} w_r^k > K-1$. Then
\begin{align*}
\sum_{s=K}^t \frac{w_s^k}{\sqrt{N_s^k}}
= \sum_{s=K}^{t_0^k-1} \frac{w_s^k}{\sqrt{N_s^k}}
+ \sum_{s=t_0^k}^t \frac{w_s^k}{\sqrt{N_s^k}} 
\leq \sum_{s=K}^{t_0^k-1} w_s^k
+ \sum_{s=t_0^k}^t \frac{w_s^k}{\sqrt{N_s^k}} 
&\leq K + \sum_{s=t_0^k}^t \frac{w_s^k}{\sqrt{N_s^k}}\: .
\end{align*}
By the tracking property of Lemma~\ref{lem:tracking},
\begin{align*}
\sum_{s=t_0^k}^t \frac{w_s^k}{\sqrt{N_s^k}}
&\leq \sum_{s=t_0^k}^t \frac{w_s^k}{\sqrt{\sum_{r=1}^s w_r^k - (K-1)}} \: .
\end{align*}
By Lemma~\ref{lem:sum_w/sqrt(N)},
\begin{align*}
\sum_{s=t_0^k}^t \frac{w_s^k}{\sqrt{\sum_{r=1}^s w_r^k - (K-1)}}
&\leq 2\sqrt{\sum_{s=1}^t w_s^k - (K-1)} - 2 \sqrt{\sum_{s=1}^{t_0^k} w_s^k - (K-1)}
\leq 2\sqrt{\sum_{s=1}^t w_s^k}\: .
\end{align*}
Putting all these computations together, we obtain
\begin{align*}
\sum_{k=1}^K\sum_{s=K}^t \frac{w_s^k}{\sqrt{N_s^k}}
&\leq K^2 + 2\sum_{k=1}^K\sqrt{\sum_{s=1}^t w_s^k}
\leq K^2 + 2\sqrt{K t} \: .
\end{align*}

We now prove the inequality on the right. For $s$ such that $N_{s-1}^k \geq 1$, we have $N_{s-1}^k \geq \frac{1}{2} N_s^k$. We remark that this is true for all $s\geq K$, apply it to the sum starting from $t_0^k$, and obtain the wanted inequality.
\end{proof}

\section{Sample complexity proof}\label{sec:sample_complexity_proof}

\subsection{Upper confidence bounds}\label{sec:ucb}

At stage $t$, we compute the empirical mean vector $\hat{\vmu}_{t-1}$ and the mixed strategies of the two players $\w_t$ and $\q_t$. 
A concentration event ensures that for all $k\in[K]$, both $\mu^k$ and $\hat{\mu}_{t-1}^k$ belong to an interval $[a_t^k, b_t^k]$.
We introduce two types of coordinate-wise upper confidence bounds (UCB). The first type is a vector $U_t \in\R^K$ such that
\begin{align*}
\forall k\in[K], \forall \xi^k \in [a_t^k, b_t^k], \: U_t^k
&\geq \ex_{\vlambda\sim\q_t} d(\xi^k, \lambda^k) \: .
\end{align*}
The second type is a function of $\vlambda$, $U_t^k(\vlambda)$ such that
\begin{align*}
\forall k\in[K], \forall \vlambda\in\mathcal{M}, \forall \xi^k \in [a_t^k, b_t^k], \: U_t^k(\vlambda)
&\geq  d(\xi^k, \lambda^k) \: .
\end{align*}

Let $[\alpha_t^k, \beta_t^k]$ be the intersection of $[\mu_{\min},\mu_{\max}]$ and the interval $\{\xi\in\Theta \: : \: d(\hat{\mu}_{t-1}^k, \xi) \leq \frac{f(t-1)}{N_{t-1}^k}\}$, where $f$ is defined in Definition~\ref{def:f} in section \ref{sec:concentration}.

Let $[a_t^t,b_t^k] = [\mu_{\min},\mu_{\max}] \cap [\hat{\mu}_{t-1}^k - \sqrt{2\sigma^2\frac{f(t-1)}{N_{t-1}^k}},\hat{\mu}_{t-1}^k + \sqrt{2\sigma^2\frac{f(t-1)}{N_{t-1}^k}}]$. 

We consider the following UCBs.
\begin{enumerate}
\item ${U_t^k}^{(1)} = \max\left\{\frac{f(t-1)}{N_{t-1}^k}, \max_{\xi \in [\alpha_t^k, \beta_t^k]} \ex_{\vlambda\sim\q_t}d(\xi, \lambda^k)  \right\}$.
\item ${U_t^k}^{(2)} = \max\left\{\frac{f(t-1)}{N_{t-1}^k}, \max_{\xi \in [a_t^k, b_t^k]} \ex_{\vlambda\sim\q_t}d(\xi, \lambda^k) \right\}$.
\item ${U_t^k}^{(1)}(\vlambda) = \max\left\{\frac{f(t-1)}{N_{t-1}^k}, \max_{\xi \in [\alpha_t^k, \beta_t^k]} d(\xi, \lambda^k)  \right\}$.
\item ${U_t^k}^{(2)}(\vlambda) = \max\left\{\frac{f(t-1)}{N_{t-1}^k}, \max_{\xi \in [a_t^k, b_t^k]} d(\xi, \lambda^k) \right\}$.
\end{enumerate}

The UCBs indexed by $(2)$ are larger but potentially easier to compute that the ones indexed by $(1)$, since $a_t^k$ and $b_t^k$ are easier to compute than $\alpha_t^k$ and $\beta_t^k$. The next lemma simplifies the computation of the UCBs.

\begin{lemma}
In all the UCBs introduced, the maximum over the interval is attained at one of the two extremal points.
\end{lemma}
\begin{proof}
We need to prove that a function of the form $\xi \mapsto \ex_{\vlambda\sim\q}d(\xi, \lambda^k)$ attains its maximum at an extremity of any interval. That function has derivative equal to $\phi'(\xi) - \ex_{\vlambda\sim \q} \phi'(\lambda^k)$. Since $\phi'$ is increasing, that derivative is negative below a point and positive afterwards. Hence the function is decreasing then increasing. We obtain that its maximum is indeed attained on an extremity of the interval.
\end{proof}

\begin{lemma}\label{lem:comparison_of_ucbs}
For all $k\in[K]$, all $t\in\N$, ${U_t^k}^{(1)} \leq {U_t^k}^{(2)}$. Furthermore for all $\vlambda\in\mathcal{M}$, ${U_t^k}^{(1)}(\vlambda) \leq {U_t^k}^{(2)}(\vlambda)$.
\end{lemma}
\begin{proof}
By the sub-Gaussian assumption~\ref{ass:sub-Gaussian}, $[\alpha_t^k, \beta_t^k] \subseteq [a_t^k, b_t^k]$.
\end{proof}

\begin{lemma}
${U_t^k}^{(1)}$ and ${U_t^k}^{(2)}$ verify $\forall \xi \in [\alpha_t^k, \beta_t^k], \: U_t^k \geq \ex_{\vlambda\sim\q_t} d(\xi, \lambda^k)$.
${U_t^k}^{(1)}(\vlambda)$ and ${U_t^k}^{(2)}(\vlambda)$ verify $\forall \xi \in [\alpha_t^k, \beta_t^k], \: U_t^k \geq d(\xi, \lambda^k)$.
\end{lemma}
\begin{proof}
It is true for ${U_t^k}^{(1)}$ and ${U_t^k}^{(1)}(\vlambda)$ by definition and true for ${U_t^k}^{(2)}$ and ${U_t^k}^{(2)}(\vlambda)$ by Lemma~\ref{lem:comparison_of_ucbs}.
\end{proof}

The analysis will proceed identically with ${U_t^k}^{(1)}$ or ${U_t^k}^{(2)}$ (resp. ${U_t^k}^{(1)}(\vlambda)$ or ${U_t^k}^{(2)}(\vlambda)$), which will be denoted simply by $U_t^k$ (resp. $U_t^k(\vlambda)$). The following lemma is an immediate consequence of the definition.

\begin{lemma}\label{lem:exploration_term_lower_bound}
All UCBs presented verify $U_t^k \geq \frac{f(t-1)}{N_{t-1}^k}$ (resp. $U_t^k(\vlambda) \geq \frac{f(t-1)}{N_{t-1}^k}$).
\end{lemma}
This lower bound is the reason the UCBs are computed as the maximum of some expression and $\frac{f(t-1)}{N_{t-1}^k}$. But for ${U_t^k}^{(2)}$ and ${U_t^k}^{(2)}(\vlambda)$, that lower bound is also obtained automatically as soon as $[a_t^k, b_t^k] \subseteq [\mu_{\min}, \mu_{\max}]$. Indeed in that case

${U_t^k}^{(2)} \geq  \min\{d(\hat{\mu}_{t-1}^k-\sqrt{2\sigma^2\frac{f(t-1)}{N_{t-1}^k}}, \hat{\mu}_{t-1}^k), d(\hat{\mu}_{t-1}^k+\sqrt{2\sigma^2\frac{f(t-1)}{N_{t-1}^k}}, \hat{\mu}_{t-1}^k)\}$ .

From the sub-Gaussian assumption, they are both bigger than $\frac{f(t-1)}{N_{t-1}^k}$.

\subsection{Saddle point algorithms}

Let $\Lambda$ be a subset of $\mathcal{M}$.

\begin{definition}\label{def:optimistic_saddle_point}
In the context of this proof, an algorithm playing sequences $(\w_s,\q_s)_{s\leq t} \in (\triangle_K \times \pr(\Lambda))^{[t]}$ is said to be an approximate optimistic saddle point algorithm with slack $x_t$ if
\begin{align*}
\inf_{\vlambda \in \Lambda} \sum_{s=1}^t \sum_{k=1}^K w_s^k d(\hat{\mu}_{s-1}^k, \lambda^k)
&\geq \max_k \sum_{s=1}^t \ex_{\vlambda\sim\q_s} U_s^k(\vlambda) - x_t \: , \\
\mbox{or} \qquad \inf_{\vlambda \in \Lambda} \sum_{s=1}^t \sum_{k=1}^K w_s^k d(\hat{\mu}_{s-1}^k, \lambda^k)
&\geq \max_k \sum_{s=1}^t U_s^k - x_t \: .
\end{align*}
\end{definition}

We now show two ways to prove that a procedure is an approximate optimistic saddle point algorithm, introducing either upper bounds $U_t^k(\vlambda)$ or $U_t^k$.

\paragraph{Introduce UCBs, then use a saddle point property.}

We can start by replacing $d(\hat{\mu}_{s-1}^k, \lambda^k)$ by an UCB $U_s^k(\vlambda)$. Let $C_s^k = \sup_{\vlambda \in \Lambda} (U_s^k(\vlambda) - d(\hat{\mu}_{s-1}^k, \lambda^k))$.

\begin{align*}
\inf_{\vlambda \in \Lambda} \sum_{s=1}^t \sum_{k=1}^K w_s^k d(\hat{\mu}_{s-1}^k, \lambda^k)
&\geq \inf_{\vlambda \in \Lambda} \sum_{s=1}^t \sum_{k=1}^K w_s^k U_s^k(\vlambda) - \sum_{s=1}^t \sum_{k=1}^K w_s^k C_s^k \: .
\end{align*}

Consider the following ``optimistic'' zero-sum games, indexed by $t\in\N$: to actions $(k,\vlambda)\in[K]\times\Lambda$ corresponds a reward of $U_t^k(\vlambda)$ for the $k$-player.

An iterative saddle point algorithm attains an $(R_t^\vlambda, R_t^k)$ equilibrium at time $t$ on that sequence if
\begin{align*}
\inf_{\vlambda \in \Lambda} \sum_{s=1}^t \sum_{k=1}^K w_s^k U_s^k(\vlambda) + R_t^\vlambda
\geq  \sum_{s=1}^t \sum_{k=1}^K w_s^k \ex_{\vlambda\sim\q_s}U_s^k(\vlambda)
\geq \max_{k\in[K]} \sum_{s=1}^t \ex_{\vlambda\sim\q_s} U_s^k(\vlambda) - R_t^k\: .
\end{align*}

The notations $R_t^\vlambda$ and $R_t^k$ reflect a common strategy to attain such an equilibrium: instantiate two regret minimization algorithms for the two players, with linear losses $\ell_t^\w(\w) = -\ex_{\vlambda\sim\q_t} \sum_{k=1}^K w^k U_t^k(\vlambda)$ and $\ell_t^\vlambda(\q) = \ex_{\vlambda\sim\q} \sum_{k=1}^K w_t^k U_t^k(\vlambda)$. If we do so, the left and right inequalities are the regret properties of the algorithm for $\vlambda$ and $k$ respectively.
At that point we have
\begin{align*}
\inf_{\vlambda \in \Lambda} \sum_{s=1}^t \sum_{k=1}^K w_s^k d(\hat{\mu}_{s-1}^k, \lambda^k)
&\geq \max_{k\in[K]} \sum_{s=1}^t \ex_{\vlambda\sim\q_s} U_s^k(\vlambda) - R_t^\vlambda - R_t^k - \sum_{s=1}^t \sum_{k=1}^K w_s^k C_s^k\ .
\end{align*}
We obtain the desired property with $x_t = R_t^\vlambda + R_t^k + \sum_{s=1}^t \sum_{k=1}^K w_s^k C_s^k$ .

\paragraph{Use a regret property for $\vlambda$, then introduce UCBs, then use a regret property for $k$.}

We take here for the $\vlambda$-player a regret minimization algorithm for the loss $\ell_t^\vlambda(\q) = \sum_{k=1}^K w_t^k \ex_{\vlambda\sim\q}d(\hat{\mu}_{t-1}^k, \lambda^k)$. It verifies
\begin{align*}
\inf_{\vlambda \in \Lambda} \sum_{s=1}^t \sum_{k=1}^K w_s^k d(\hat{\mu}_{s-1}^k, \lambda^k)
\geq  \sum_{s=1}^t \sum_{k=1}^K w_s^k \ex_{\vlambda\sim\q_s}d(\hat{\mu}_{s-1}^k, \lambda^k) - R_t^\vlambda \: .
\end{align*}
We now introduce UCBs $U_s^k$,
\begin{align*}
\sum_{s=1}^t \sum_{k=1}^K w_s^k \ex_{\vlambda\sim\q_s}d(\hat{\mu}_{s-1}^k, \lambda^k)
&\geq \sum_{s=1}^t \sum_{k=1}^K w_s^k U_s^k
	- \sum_{s=1}^t \sum_{k=1}^K w_s^k C_s^k\ .
\end{align*}

The $k$-player uses a regret minimization algorithm for the loss $\ell_t^k(\w) = -\sum_{k=1}^K w_s^k U_s^k$, with regret $R_t^k$. Let $C_s^k = U_s^k - \ex_{\vlambda\sim\q_s}d(\hat{\mu}_{t-1}^k, \lambda^k)$ .
\begin{align*}
\sum_{s=1}^t \sum_{k=1}^K w_s^k U_s^k
&\geq \max_k \sum_{s=1}^t U_s^k
	-  R_t^k  \: .
\end{align*}
We obtain the desired property with  $x_t = R_t^\vlambda + R_t^k + \sum_{s=1}^t \sum_{k=1}^K w_s^k C_s^k$ .

\subsection{Concentration arguments}

Concentration event:
$
\mathcal{E}_t
= \left\{ \forall s \leq t, \forall k\in[K], \: d(\hat{\mu}_s^k, \mu^k) \leq \frac{f(t^{\nicefrac{1}{(1+b)}})}{N_s^k} \right\}
$ .

\begin{lemma}\label{lem:from_mu_to_hat_mu_s}
Under the event $\mathcal{E}_t$, for all $s\in[t]$, $k\in[K]$ and $\vlambda \in \mathcal{M}$,
\begin{align*}
|d(\mu^k, \lambda^k) - d(\hat{\mu}_{s-1}^k, \lambda^k)| &\leq L\sqrt{2\sigma^2 \frac{f(t^{\nicefrac{1}{(1+b)}})}{N_{s-1}^k}} \: .
\end{align*}
\end{lemma}
\begin{proof}
Use the Lipschitz property of $x\mapsto d(x,y)$, then the sub-Gaussian assumption and finally the definition of $\mathcal{E}_t$.
\end{proof}

\begin{lemma}
Let $C_s^k = \max\left\{ 2L\sqrt{2\sigma^2\frac{f(\max\{s-1, t^{\nicefrac{1}{(1+b)}}\})}{N_{s-1}^k}} , \frac{f(\max\{s-1, t^{\nicefrac{1}{(1+b)}}\})}{N_{s-1}^k} \right\}$. Let $\alpha_s^k$ and $\beta_s^k$ be defined as in section~\ref{sec:ucb}. Under the event $\mathcal{E}_t$, for all $s\in[t]$, all $\vlambda \in \mathcal{M}$,
\begin{align*}
\sup_{\xi\in[\alpha_s^k, \beta_s^k]}(U_s^k(\vlambda) - d(\xi, \lambda^k)) &\leq C_s^k \: ,\\
\sup_{\xi\in[\alpha_s^k, \beta_s^k]}(U_s^k - \ex_{\vlambda\sim\q_s}d(\xi , \lambda^k)) &\leq C_s^k \: .
\end{align*}
\end{lemma}

\begin{proof}
Let $u_s^k = \hat{\mu}_{s-1}^k - \sqrt{2\sigma^2\frac{f(\max\{s-1, t^{\nicefrac{1}{(1+b)}}\})}{N_{s-1}^k}}$ and $v_s^k = \hat{\mu}_{s-1}^k + \sqrt{2\sigma^2\frac{f(\max\{s-1, t^{\nicefrac{1}{(1+b)}}\})}{N_{s-1}^k}}$.

Under the event $\mathcal{E}_t$, for all $s\in[t]$, we have $\hat{\mu}_{s-1}^k,\mu^k \in [u_s^k, v_s^k]$. $U_s^k$ is defined as the maximum of $\frac{f(s-1)}{N_{s-1}^k}$ and a maximum over an interval which is contained in $[u_s^k, v_s^k]$. If $U_s^k$ is equal to the latter,
\begin{align*}
\sup_{\xi\in[\alpha_s^k, \beta_s^k]}(U_s^k - \ex_{\vlambda\sim\q_s}d(\xi , \lambda^k))
&\leq \sup_{\eta, \xi \in [u_s^k, v_s^k]} |\ex_{\vlambda\sim\q_s}d(\eta, \lambda^k) - \ex_{\vlambda\sim\q_s}d(\xi, \lambda^k)|\\
&\leq L |u_s^k - v_s^k|
\leq 2L\sqrt{2\sigma^2\frac{f(\max\{s-1, t^{\nicefrac{1}{(1+b)}}\})}{N_{s-1}^k}} \: .
\end{align*}
If $U_s^k=\frac{f(s-1)}{N_{s-1}^k}$, then $\sup_{\xi\in[\alpha_s^k, \beta_s^k]}(U_s^k - \ex_{\vlambda\sim\q_s}d(\xi , \lambda^k))
\leq  U_s^k = \frac{f(s-1)}{N_{s-1}^k}$ .

Same computations for $U_s^k(\vlambda)$, without expectations.
\end{proof}

\begin{lemma}\label{lem:sum_C_s}
\begin{align*}
\sum_{s=K+1}^t \sum_{k=1}^K w_s^k C_s^k
\leq 2L\sqrt{2\sigma^2f(t)}(K^2 + 2\sqrt{2Kt}) + f(t)(K^2 + 2K\log(t/K)) \: .
\end{align*}
\end{lemma}

\begin{proof}
Since $C_s^k$ is the maximum of two quantities, it is smaller than their sum. By Lemma~\ref{lem:concentration_term_w/sqrt(N)},
\begin{align*}
\sum_{s=K+1}^t \sum_{k=1}^K w_s^k 2L\sqrt{2\sigma^2\frac{f(\max\{s-1, t^{\nicefrac{1}{(1+b)}}\})}{N_{s-1}^k}}
&\leq 2L\sqrt{2\sigma^2f(t)}\sum_{s=K+1}^t \sum_{k=1}^K  \frac{w_s^k}{\sqrt{N_{s-1}^k}}\\
&\leq 2L\sqrt{2\sigma^2f(t)}(K^2 + 2\sqrt{2Kt}) \: ,
\end{align*}
Similarly,
\begin{align*}
\sum_{s=K+1}^t \sum_{k=1}^K w_s^k \frac{f(\max\{s-1, t^{\nicefrac{1}{(1+b)}}\})}{N_{s-1}^k}
&\leq f(t)\sum_{s=K+1}^t \sum_{k=1}^K  \frac{w_s^k}{N_{s-1}^k}\\
&\leq f(t)(K^2 + 2K\log(t/K)) \: ,
\end{align*}
\end{proof}

\begin{lemma}\label{lem:from_hat_mu_to_mu}
Under $\mathcal{E}_t$, for any $\vlambda\in\mathcal{M}$,
\begin{align*}
\sum_{k=1}^K N_t^k d(\hat{\mu}_t^k, \lambda^k)
&\geq \sum_{k=1}^K N_t^k d(\mu^k, \lambda^k) - L\sqrt{2\sigma^2 K t f(t)} \: .
\end{align*}
\end{lemma}
\begin{proof}
By the Lipschitzness assumption,
\begin{align*}
\sum_{k=1}^K N_t^k d(\hat{\mu}_t^k, \lambda^k)
&\geq \sum_{k=1}^K N_t^k d(\mu^k, \lambda^k) - L\sum_{k=1}^K N_t^k |\hat{\mu}_t^k - \mu^k| \: .
\end{align*}
Using the sub-Gaussian hypothesis, under $\mathcal{E}_t$, $|\hat{\mu}_t^k - \mu^k| \leq \sqrt{2\sigma^2 d(\hat{\mu}_t^k, \mu^)} \leq \sqrt{2\sigma^2\frac{f(t)}{N_t^k}}$.
\begin{align*}
\sum_{k=1}^K N_t^k d(\hat{\mu}_t^k, \lambda^k)
&\geq \sum_{k=1}^K N_t^k d(\mu^k, \lambda^k) - L\sum_{k=1}^K N_t^k \sqrt{2\sigma^2\frac{f(t)}{N_t^k}}\\
&=    \sum_{k=1}^K N_t^k d(\mu^k, \lambda^k) - L\sqrt{2\sigma^2 f(t)}\sum_{k=1}^K \sqrt{N_t^k}\\
&\geq \sum_{k=1}^K N_t^k d(\mu^k, \lambda^k) - L\sqrt{2\sigma^2 K t f(t)} \: .
\end{align*}
\end{proof}

\subsection{The candidate answer}

The data seen before time $t$ is summarized in the vector $\hat{\vmu}_{t-1}\in\Theta^K$. That vector does not in general belong to $\mathcal{M}$.

Our algorithm finds any point in the intersection of $\mathcal{M}$ and the confidence box around $\hat{\vmu}_{t-1}$.
The point obtained is denoted by $\vmu^\mathcal{M}_{t-1}$ and verifies that for all $k\in[K]$, $d(\hat{\mu}_{t-1}^k, \mu_{t-1}^{\mathcal{M} k})\leq \frac{f(t-1)}{N_{t-1}^k}$  .
The candidate answer used at time $t$ is then $i_t = i^*(\vmu^\mathcal{M}_{t-1})$.

\subsection{When the candidate answer is not the correct answer}

\paragraph{Chernoff information.} For $x,y\in\Theta$, let $\ch(x,y) = \inf_{u\in\Theta}(d(u,x) + d(u,y))$ be the Chernoff information between $x$ and $y$.

\begin{assumption}\label{ass:chernoff_separation}
There exists $\varepsilon>0$ such that for all $\vlambda \in \neg i^*(\vmu)$, there exists $k\in[K]$ such that $\ch(\lambda^k, \mu^k) \geq \varepsilon$.
\end{assumption}

If the distributions are sub-Gaussian with parameter $\sigma^2$, then $\ch(x,y) \geq \frac{(x-y)^2}{8\sigma^2}$ and that assumption is true for every $\vmu\in\mathcal{M}$ with $D_\vmu>0$. i.e. Assumption~\ref{ass:sub-Gaussian} implies Assumption~\ref{ass:chernoff_separation}.

\begin{lemma}\label{lem:different_answers_means_a_small_N_v2}
Suppose that Assumption~\ref{ass:chernoff_separation} holds for $\vmu\in\mathcal{M}$ and that for all $k\in[K]$, $d(\hat{\mu}_{t-1}^k,\mu^k) \leq \frac{\log(t-1)}{N_{t-1}^k}$. If $i^*(\vmu^\mathcal{M}_{t-1})\neq i^*(\vmu)$ then there exists $j\in[K]$ such that $\frac{f(t-1)}{N_{t-1}^j} \geq \frac{\varepsilon}{2}$.
\end{lemma}
\begin{proof}
If $i^*(\tilde{\vmu}_{t-1})\neq i^*(\vmu)$ then $\vmu^\mathcal{M}_{t-1}$ belongs to the set $\neg i^*(\vmu)$.

By Assumption~\ref{ass:chernoff_separation}, there exists $j\in[K]$ such that $\ch(\mu^k, {\mu^\mathcal{M}_{t-1}}^k) \geq \varepsilon$. By definition of $\ch$ as an infimum over $\Theta$, it is smaller than $d(\hat{\mu}_{t-1}^j, \mu^j) + d(\hat{\mu}_{t-1}^j, {\mu^\mathcal{M}_{t-1}}^j)$. That sum is then bigger than $\varepsilon$, with consequence that either $d(\hat{\mu}_{t-1}^j, \mu^j) \geq \varepsilon/2$ or $d(\hat{\mu}_{t-1}^j, {\mu^\mathcal{M}_{t-1}}^j) \geq \varepsilon/2$.

If $d(\hat{\mu}_{t-1}^j, \mu^j) \geq \varepsilon/2$, then by hypothesis, $\frac{f(t-1)}{N_{t-1}^j} \geq d(\hat{\mu}_{t-1}^j, \mu^j) \geq \varepsilon/2$.

Otherwise $d(\hat{\mu}_{s-1}^j, {\mu^\mathcal{M}_{t-1}}^j) \geq \varepsilon/2$. By definition of $\mu^\mathcal{M}_{t-1}$, $\frac{f(t-1)}{N_{t-1}^j} \geq d(\hat{\mu}_{t-1}^j, {\mu^\mathcal{M}_{t-1}}^j)$. We proved that $\frac{f(t-1)}{N_{t-1}^j} \geq \frac{\varepsilon}{2}$.
\end{proof}

\paragraph{Linear increase in information.}

For $i\in\mathcal{I}$, let $n_i(t)$ be the number of stages $s\leq t$ in which $i_s = i$. To shorten notations, let $i^*=i^*(\vmu)$. The goal of this section is to find a lower bound for $n_{i^*}(t)$. We do it by showing that when the answer $i_s$ is not the correct one, a quantity is linearly increasing, while at the same time being $O(\sqrt{t})$ by a concentration argument. Hence the number of time steps this can happen is also $O(\sqrt{t})$.

Using that $\vmu \in \neg i_s$,
\begin{align*}
\sum_{s\leq t, i_s\neq i^*} \sum_{k=1}^K w_s^k d(\hat{\mu}_{s-1}^k, \mu^k) 
&\geq \sum_{i\in\mathcal{I}\setminus \{i^*\}} \inf_{\vlambda\in\neg i} \sum_{s\leq t, i_s=i} \sum_{k=1}^K w_s^k d(\hat{\mu}_{s-1}^k, \lambda^k)\: .
\end{align*}
Let $\varepsilon_t$ be the quantity on the left, which will be small by a concentration argument.

The algorithm used when $i_s=i$ is an optimistic approximate saddle point algorithm with slack $R_{n_i(t)}^k + R_{n_i(t)}^\vlambda + \sum_{s\leq t, i_s=i} \sum_{k=1}^K w_s^k C_s^k$. Hence we have
\begin{align*}
\varepsilon_t
\geq \sum_{i\in\mathcal{I}\setminus \{i^*\}} \max_k \sum_{s\leq t, i_s=i} U_s^k
	- \sum_{i\in\mathcal{I}\setminus \{i^*\}} (R_{n_i(t)}^k + R_{n_i(t)}^\vlambda) - \sum_{s\leq t, i_s\neq i^*}\sum_{k=1}^K w_s^k C_s^k \: .
\end{align*}
Note: if UCBs of the form $U_s^k(\vlambda)$ are used instead of $U_s^k$, replace $U_s^k$ by $\ex_{\vlambda\sim\q_s}U_s^k(\vlambda)$ here and in the following expressions.

For fixed $i\in\mathcal{I}\setminus \{i^*\}$, we now shos that the quantity $\max_k \sum_{s\leq t, i_s=i} U_s^k$ increases linearly with the number of terms of the sum, $n_i(t)$.
We proved in Lemma~\ref{lem:exploration_term_lower_bound} that for all $s\in\N$ and $k\in[K]$, $U_s^k \geq \frac{f(s-1)}{N_{s-1}^k}$.
When the event $\mathcal{E}_t$ holds, for all $s\in[t^{\nicefrac{1}{(1+b)}},t]$ with $i_s\neq i^*$, there is a $j_s\in[K]$ such that
$U_{s}^{j_s} \geq \varepsilon/2$ by Lemma~\ref{lem:different_answers_means_a_small_N_v2}.

Let $t'$ be the last term of the sum and suppose that $t'>\sqrt{t}$. Let $j$ be such that $U_{t'}^j \geq \varepsilon/2$. Then for all $s\in[\lceil t^{\nicefrac{1}{(1+b)}}\rceil, t']$,
\begin{align*}
\frac{f(s-1)}{N_{s-1}^j}
\geq \frac{f(s-1)}{N_{t'-1}^j}
= \frac{f(s-1)}{f(t'-1)} \frac{f(t'-1)}{N_{t'-1}^j}
\geq \frac{f(t^{\nicefrac{1}{(1+b)}})}{f(t)} \varepsilon/2 \: .
\end{align*}
For $t>e$,  $\frac{f(t^{\nicefrac{1}{(1+b)}})}{f(t)} \geq \frac{1}{3(1+b)}$. Let $C_b = 1/(3(1+b))$.

Hence for that arm $j$, $\sum_{s\leq t, i_s=i} U_s^j \geq C_b \varepsilon (n_i(t) - n_i(t^{\nicefrac{1}{(1+b)}}))/2$.

We conclude that the maximum over $k$ of the sums is also bigger than this quantity.
We have shown
\begin{align*}
\varepsilon_t
&\geq \sum_{i\in\mathcal{I}\setminus \{i^*\}} \frac{C_b \varepsilon}{2} (n_i(t) -n_i(t^{\nicefrac{1}{(1+b)}}))
- \sum_{i\in\mathcal{I}\setminus \{i^*\}} (R_{n_i(t)}^k + R_{n_i(t)}^\vlambda)
- \sum_{s=K+1}^t \sum_{k=1}^K w_s^k C_s^k \\
&\geq \frac{C_b \varepsilon}{2} (t - t^{\nicefrac{1}{(1+b)}} - n_{i^*}(t))
- (|\mathcal{I}|-1) (R_t^k + R_t^\vlambda)
- \sum_{s=K+1}^t \sum_{k=1}^K w_s^k C_s^k \: .
\end{align*}

If $n\mapsto R^k_n$ and $n\mapsto R^\vlambda_n$ are concave (for example regret proportional to $\sqrt{n}$), the regret term has the form $(|\mathcal{I}|-1) (R_{(t-n_{i^*})/(|I|-1)}^k + R_{(t-n_{i^*})/(|I|-1)}^\vlambda)$.

By concentration,
\begin{align*}
\varepsilon_t \leq f(t^{\nicefrac{1}{(1+b)}}) \sum_{s=1}^t \sum_{k=1}^K\frac{w_s^k}{N_{s-1}^t} \leq f(t)(K^2 + 2K\log(t/K)) \: .
\end{align*}

We proved
\begin{align}\label{eq:number_rounds_with_correct_answer_is_linear}
n_{i^*}(t) \geq t - t^{\nicefrac{1}{(1+b)}} - \frac{2}{C_b \varepsilon}\left( (|\mathcal{I}|-1) (R_t^k{+}R_t^\vlambda) + f(t)(K^2 {+} 2K\log(t/K)) + \hspace{-6pt} \sum_{s=K+1}^t \sum_{k=1}^K w_s^k C_s^k \right)
\end{align}

\subsection{When the candidate answer is the correct answer}\label{sec:proof_correct_answer}

Let $t'\leq t$ be the last round in which $i_{t'}=i^*$ before the algorithm stops. Then $t'\geq n_{i^*}(t)$, we have $i_{t'} = i^*$ and $n_{i^*}(t') = n_{i^*}(t)$.
\begin{align*}
\beta(t, \delta)
\geq \beta(t',\delta)
\geq \inf_{\vlambda\in\neg i_{t'}}\sum_{k=1}^K N_{t'}^k d(\hat{\mu}_{t'}^k, \lambda^k)
&=    \inf_{\vlambda\in\neg i^*} \sum_{k=1}^K N_{t'}^k d(\hat{\mu}_{t'}^k, \lambda^k)\\
&\geq \inf_{\vlambda\in\neg i^*} \sum_{k=1}^K N_{t'}^k d(\mu^k, \lambda^k) - L\sqrt{2\sigma^2 Kt f(t)}\: .
\end{align*}
Using the tracking Lemma~\ref{lem:tracking}, then concentration Lemma~\ref{lem:from_mu_to_hat_mu_s},
\begin{align*}
\beta(t, \delta)
&\geq \inf_{\vlambda\in\neg i^*} \sum_{s=1}^{t'} \sum_{k=1}^K w_s^k d(\mu^k, \lambda^k) - KD - L\sqrt{2\sigma^2 Kt f(t)}\\
&\geq \inf_{\vlambda\in\neg i^*} \sum_{s=K+1}^{t'} \sum_{k=1}^K w_s^k d(\hat{\mu}_{s-1}^k, \lambda^k)\\
&\qquad - L\sqrt{2\sigma^2f(t)}\sum_{s=K+1}^t \sum_{k=1}^K \frac{w_s^k}{\sqrt{N_{s-1}^k}} - KD - L\sqrt{2\sigma^2 Kt f(t)} \\
&\geq \inf_{\vlambda\in\neg i^*} \sum_{s=K+1}^{t'} \sum_{k=1}^K w_s^k d(\hat{\mu}_{s-1}^k, \lambda^k) \\
&\qquad- 2L\sqrt{2\sigma^2f(t)}(K^2 + 2\sqrt{2K t}) - KD - L\sqrt{2\sigma^2 Kt f(t)}
\end{align*}
We drop the rounds in which $i_s\neq i^*$.
\begin{align*}
\beta(t, \delta)
&\geq \inf_{\vlambda\in\neg i^*} \sum_{K+1\leq s\leq t', i_s=i^*} \sum_{k=1}^K w_s^k d(\hat{\mu}_{s-1}^k, \lambda^k)\\
&\qquad- 2L\sqrt{2\sigma^2f(t)}(K^2 {+} 2\sqrt{2K t}) - KD - L\sqrt{2\sigma^2 Kt f(t)} \: .
\end{align*}

The algorithm used is an optimistic approximate saddle point algorithm with slack $R_t^\vlambda + R_t^k + \sum_{s=1}^t \sum_{k=K+1}^K w_s^k C_s^k$ :
\begin{align*}
\inf_{\vlambda\in\neg i^*} \sum_{s\leq t', i_s=i^*} \sum_{k=1}^K w_s^k d(\hat{\mu}_{s-1}^k, \lambda^k)
&\geq \max_k \sum_{s\leq t', i_s=i^*} U_s^k - (R_t^\vlambda + R_t^k + \sum_{s=K+1}^t \sum_{k=1}^K w_s^k C_s^k) \ .
\end{align*}

Let $A_t = \sum_{s=K+1}^t \sum_{k=1}^K w_s^k C_s^k + 2L\sqrt{2\sigma^2f(t)}(K^2 + 2\sqrt{2K t}) + KD + L\sqrt{2\sigma^2 Kt f(t)}$. We obtain
\begin{align*}
\beta(t, \delta)
&\geq \max_k \sum_{K < s\leq t', i_s=i^*} U_s^k - R_t^k - R_t^\vlambda - A_t \: .
\end{align*}

Let $t_b = t^{\nicefrac{1}{(1+b)}}$. Since $U_t$ is a coordinate-wise upper confidence bound when concentration holds (for $s\geq t^{\nicefrac{1}{(1+b)}}$), we have
\begin{align*}
\beta(t, \delta)
&\geq \max_k \sum_{t_b\leq s\leq t', i_s=i^*} \ex_{\vlambda\sim\q_s}d(\mu^k, \lambda^k) - R_t^k - R_t^\vlambda - A_t\\
&=    (n_{i^*}(t')-t_b) \max_k \frac{1}{(n_{i^*}(t')-t_b)}\sum_{t^{\nicefrac{1}{(1+b)}} \leq s\leq t', i_s=i^*} \ex_{\vlambda\sim\q_s}d(\mu^k, \lambda^k) - R_t^k - R_t^\vlambda - A_t\\
&\geq (n_{i^*}(t')-t_b) \inf_{\q\in\pr(\neg i^*)} \max_k \ex_{\vlambda\sim\q} d(\mu^k, \lambda^k) - R_t^k - R_t^\vlambda - A_t\\
&=    (n_{i^*}(t')-t^{\nicefrac{1}{(1+b)}}) D_\vmu - R_t^k - R_t^\vlambda - A_t \: .
\end{align*}

$t'$ is such that $n_{i^*}(t') = n_{i^*}(t)$. Combining that result and the lower bound on $n_{i^*}(t)$ of equation~\eqref{eq:number_rounds_with_correct_answer_is_linear}, we have 
\begin{align}\label{eq:stopping_time_equation}
&\frac{\beta(t, \delta) + A_t + R_t^k + R_t^\vlambda}{D_\vmu}\\
&\geq t - 2 t^{\nicefrac{1}{(1+b)}} - \frac{2}{C_b\varepsilon}\left( (|\mathcal{I}|-1) (R_t^k+R_t^\vlambda) + f(t)(K^2 + 2K\log(t/K)) + \hspace{-3pt} \sum_{s=K+1}^t \sum_{k=1}^K w_s^k C_s^k \right) \: .
\end{align}

\subsection{Stopping time upper bound}

We can solve equation~\eqref{eq:stopping_time_equation} to find an upper bound for $t$ such that the algorithm does not stop.
Suppose that there exists $R>0$ such that $R_t^k+R_t^\vlambda \leq R \sqrt{Kt}$. Take $b=1$.
By Lemma~\ref{lem:sum_C_s},

$\sum_{s=K+1}^t \sum_{k=1}^K w_s^k C_s^k \leq 2L\sqrt{2\sigma^2f(t)}(K^2 + 2\sqrt{2Kt}) + f(t)(K^2 + 2K\log(t/K))$.

$A_t \leq 4L\sqrt{2\sigma^2f(t)}(K^2 + 2\sqrt{2K t}) + KD + L\sqrt{2\sigma^2 Kt f(t)} + f(t)(K^2 + 2K\log(t/K))$.

We now define
\begin{align*}
h(t) &= 2\sqrt{t} + \frac{A_t + R\sqrt{Kt}}{D_\vmu}\\
 &\: + \frac{2}{C_b\varepsilon}\left((|\mathcal{I}|-1)R\sqrt{Kt} + 2f(t)(K^2 {+} 2K\log(t/K)) + 2L\sqrt{2\sigma^2f(t)}(K^2 {+} 2\sqrt{2Kt}) \right) \: .
\end{align*}

We have that $h(t) = \mathcal{O}(\sqrt{t \log t})$ and we obtained that if $t<\tau_\delta$ then
\begin{align*}
t - h(t) \leq \frac{\beta(t,\delta)}{D_\vmu} \: .
\end{align*}

\section{Algorithms}
\subsection{Optimistic Track and Stop}

We prove that under the concentration event $\mathcal{E}_t$, there is an upper bound on $t$ such that $t<\tau_\delta$.

Let $\mathcal{C}_s = \{\vxi\in \Theta^K : \forall k\in[K], d(\hat{\mu}_{s-1}^k, \xi^k) \leq \frac{f(s-1)}{N_{s-1}^k} \}$ be a confidence region around $\hat{\vmu}_{s-1}$.

\paragraph{When $i_t\neq i^*(\vmu)$.}

Let $i\in\mathcal{I}\setminus\{i^*(\vmu)\}$. Since $i_s\neg i^*(\vmu)$ implies that $\vmu\in\neg i_s$,
\begin{align*}
\sum_{s\leq t, i_s=i} \sum_{k=1}^K w_s^k d(\hat{\mu}^k_{s-1}, \mu^k)
&\geq \inf_{\vlambda\in\neg i} \sum_{s\leq t, i_s=i} \sum_{k=1}^K w_s^k d(\hat{\mu}^k_{s-1}, \lambda^k)\ .
\end{align*}
Let $\varepsilon_t^i$ be the left hand side of that inequality. Since $\hat{\vmu}_{s-1}$ and $\vmu^{+}_s$ both belong to $\mathcal{C}_s$, we have
\begin{align*}
\sum_{s\leq t,i_s=i} \sum_{k=1}^K w_s^k d(\hat{\mu}^k_{s-1}, \lambda^k)
\geq \sum_{s\leq t,i_s=i} \sum_{k=1}^K w_s^k d(\mu^{+ k}_s, \lambda^k) - L\sqrt{2\sigma^2 f(t)}\sum_{s\leq t,i_s=i} \frac{w_s^k}{\sqrt{N_{s-1}^k}} \ .
\end{align*}
By definition of $\vmu^+_s$,
\begin{align*}
\inf_{\vlambda\in\neg i} \sum_{s\leq t, i_s=i} \sum_{k=1}^K w_s^k d(\mu^{+ k}_s, \lambda^k)
\geq \sum_{s\leq t, i_s=i} \inf_{\vlambda\in\neg i} \sum_{k=1}^K w_s^k d(\mu^{+ k}_s, \lambda^k)
&=    \sum_{s\leq t, i_s=i} D_{\vmu^+_s} \ .
\end{align*}
For $s\geq t^{1/(1+b)}$, $\vmu\in\mathcal{C}_s$ and by definition of $\vmu^+_s$, $D_{\vmu^+_s} \geq D_\vmu$. We obtain, with $n_i(t)$ the number of times with $i_s=i$ until $t$,
\begin{align*}
\sum_{i\in\mathcal{I}\setminus\{i^*(\vmu)\}} \varepsilon_t^i
&\geq (t - n_{i^*(\vmu)}(t) - t^{1/(1+b)}) D_\vmu - L\sqrt{2\sigma^2 f(t)}\sum_{s\leq t} \frac{w_s^k}{\sqrt{N_{s-1}^k}}\\
&\geq (t - n_{i^*(\vmu)}(t) - t^{1/(1+b)}) D_\vmu - L\sqrt{2\sigma^2 f(t)}(K^2 + 2\sqrt{2Kt})\ .
\end{align*}

See Lemma~\ref{lem:concentration_term_w/sqrt(N)} for that last inequality. By concentration,
\begin{align*}
\sum_{i\in\mathcal{I}\setminus\{i^*(\vmu)\}} \varepsilon_t^i
&\leq f(t) \sum_{s=1}^t \sum_{k=1}^K \frac{w_s^k}{N_{s-1}^k} \leq f(t)(K^2 + 2K\log(t/K)) \ .
\end{align*}
Finally,
\begin{align*}
n_{i^*(\vmu)}(t) \geq t - t^{1/(1+b)} - \frac{1}{D_\vmu} \left( L\sqrt{2\sigma^2 f(t)}(K^2 + 2\sqrt{2Kt}) + f(t)(K^2 + 2K\log(t/K)) \right) \ .
\end{align*}

\paragraph{When $i_t=i^*(\vmu)$.}

Let $t'\geq n_{i^*(\vmu)}(t)$ be such that $i_{t'} = i^*(\vmu)$ and $n_{i^*(\vmu)}(t') = n_{i^*(\vmu)}(t)$. Using concentration and tracking properties, as in the main sample complexity proof of Appendix~\ref{sec:proof_correct_answer},
\begin{align*}
\beta(t',\delta)
&\geq \inf_{\vlambda\in\neg i^*(\vmu)} \sum_{k=1}^K N_t^k d(\hat{\mu}_{t'}^k, \lambda^k)\\
&\geq \inf_{\vlambda\in\neg i^*(\vmu)} \sum_{k=1}^K N_t^k d(\mu^k, \lambda^k) - L\sqrt{2\sigma^2 K t f(t)}\\
&\geq \inf_{\vlambda\in\neg i^*(\vmu)} \sum_{s=1}^{t'} \sum_{k=1}^K w_s^k d(\mu^k, \lambda^k) - KD - L\sqrt{2\sigma^2 K t f(t)}\\
&\geq \inf_{\vlambda\in\neg i^*(\vmu)} \sum_{s=1}^{t'} \sum_{k=1}^K w_s^k d(\hat{\mu}^k_{s-1}, \lambda^k)\\
&\qquad- L\sqrt{2\sigma^2 f(t)}(K^2 {+} 2\sqrt{2Kt}) - KD - L\sqrt{2\sigma^2 K t f(t)}
\end{align*}
Since $\hat{\vmu}_{s-1}$ and $\vmu^{+}_s$ both belong to $\mathcal{C}_s$, we have
\begin{align*}
\inf_{\vlambda\in\neg i^*(\vmu)} \sum_{s=1}^{t'} \sum_{k=1}^K w_s^k d(\hat{\mu}^k_{s-1}, \lambda^k)
\geq \inf_{\vlambda\in\neg i^*(\vmu)} \sum_{s=1}^{t'} \sum_{k=1}^K w_s^k d(\mu^{+ k}_s, \lambda^k) - L\sqrt{2\sigma^2 f(t)}(K^2 {+} 2\sqrt{2Kt}) \ .
\end{align*}
Let $B_t =  2L\sqrt{2\sigma^2 f(t)}(K^2 + 2\sqrt{2Kt}) + KD + L\sqrt{2\sigma^2 K t f(t)}$.
\begin{align*}
\beta(t,\delta)
&\geq \inf_{\vlambda\in\neg i^*(\vmu)} \sum_{s\leq t',i_s=i^*(\vmu)} \sum_{k=1}^K w_s^k d(\mu^{+ k}_s, \lambda^k) - B_t\\
&\geq \sum_{s\leq t',i_s=i^*(\vmu)} \inf_{\vlambda\in\neg i_t} \sum_{k=1}^K w_s^k d(\mu^{+ k}_s, \lambda^k) - B_t \\
&=    \sum_{s\leq t',i_s=i^*(\vmu)} D_{\vmu^+_s} - B_t \ .
\end{align*}
For $s\geq t^{1/(1+b)}$, $\vmu \in \mathcal{C}_s$. Then by definition of $\vmu^+_s$, $D_{\vmu^+_s} \geq D_\vmu$.
\begin{align*}
\beta(t,\delta)
&\geq \sum_{t^{1/(1+b)}\leq s\leq t', i_s=i^*(\vmu)} D_\vmu - B_t\\
&=    (n_{i^*(\vmu)}(t)-t^{1/(1+b)}) D_\vmu - B_t \: .
\end{align*}

\paragraph{Putting things together.}

Let $h(t) = 3L\sqrt{2\sigma^2 f(t)}(K^2 + 2\sqrt{2Kt}) + f(t)(K^2 + 2K\log(t/K)) + KD + L\sqrt{2\sigma^2 K t f(t)}$. When the concentration event $\mathcal{E}_t$ holds, if $t<\tau_\delta$ then
\begin{align*}
\frac{\beta(t,\delta) + h(t)}{D_\vmu} \geq t - 2t^{1/(1+b)} \ .
\end{align*}
Let $T_0(\delta)$ be the maximal $t$ verifying this inequality. Then the expected sample complexity is lower than $T_0(\delta) + \frac{2eK}{a^2}$. Note that $f(t)$ depends on $a$ and $b$.

\subsection{Follow The Perturbed Leader}\label{sec:ftpl_proof}

In this section, we suppose that the rewards are bounded and we define $C>0$ such that for all times $s$ and $k\in[K]$, $|X_s^k - \hat{\mu}_{s-1}^k|\leq C$.

At stage $t$, the loss of a vector $\vlambda$ is $\ell_t(\vlambda) = d(\hat{\mu}_{t-1}^{k_t}, \lambda^{k_t})$. The only unknown quantity for the $\vlambda$-player is $k_t$. We will use the form of that loss in the way we perturb the leader.
For $\vsigma\in \R^K_+$ and $\vxi\in\Theta^K$ we define
\begin{align*}
\vlambda_t(\vsigma, \vxi)
&= \argmin_\vlambda \sum_{s=1}^{t-1} \ell_s(\vlambda) + \sum_{k=1}^K \sigma^k d(\xi^k, \lambda^k) \: .
\end{align*}

We study the expected regret of an algorithm playing $\vlambda_t(\vsigma_t, \hat{\vmu}_{t-1})$ with exponentially distributed perturbations $\vsigma_t$. Let $\q_t$ be the distribution of $\vlambda_t(\vsigma_t, \hat{\vmu}_{t-1})$.
Let $\tilde{\mu}_{t-1}^k = \frac{1}{N_{t-1}^k}\sum_{s=1}^{t-1}\hat{\mu}_{s-1}^k \mathbb{I}\{k_s=k\}$. We show in the following lemma that the point $\vlambda_t(\vsigma_t, \hat{\vmu}_{t-1})$ can be computed by the best-response oracle, as
\begin{align*}
\argmin_{\vlambda\in\Lambda} \sum_{k=1}^K (N_{t-1}^k + \sigma_t^k) d\left( \frac{N_{t-1}^k}{N_{t-1}^k + \sigma_t^k}\tilde{\mu}_{t-1}^k + \frac{\sigma_t^k}{N_{t-1}^k + \sigma_t^k}\hat{\mu}_{t-1}^k, \lambda^k \right) \: .
\end{align*}

\begin{lemma}\label{lem:oracle_can_do_FTL}
Let $(\vmu_s)_{s\in[t]}$ be $t$ points in $\Theta^K$. Then
\begin{align*}
\argmin_{\vlambda\in\Lambda} \sum_{s=1}^t d(\mu_{s}^{k_s}, \lambda^{k_s})
= \argmin_{\vlambda\in\Lambda} \sum_{k=1}^K N_t^k d(\frac{\sum_{s=1}^t \mu_s^k \mathbb{I}\{k_s=k\}}{N_t^k}, \lambda^k) \: .
\end{align*}
\end{lemma}

\begin{proof}
This is an extension of the following property:
\begin{align*}
\argmin_{\lambda} d(\mu_1, \lambda) + d(\mu_2, \lambda) = \argmin_{\lambda} d(\frac{\mu_1 + \mu_2}{2}, \lambda) \: .
\end{align*}
Indeed we can observe that fact by developing the divergence in terms of $\phi$ and observing that the terms depending on $\lambda$ are the same up to a multiplicative factor.
\begin{align*}
d(\mu_1, \lambda) + d(\mu_2, \lambda) &= \phi(\mu_1) + \phi(\mu_2) - 2 \left(\phi(\lambda) + \phi'(\lambda) (\frac{\mu_1 + \mu_2}{2} - \lambda)\right) \ ,\\
d(\frac{\mu_1 + \mu_2}{2}, \lambda) &= \phi(\frac{\mu_1 + \mu_2}{2}) - \left(\phi(\lambda) + \phi'(\lambda) (\frac{\mu_1 + \mu_2}{2} - \lambda)\right)\: .
\end{align*}
\end{proof}

\begin{theorem}\label{th:regret_FTPL}
The expected regret of the FTPL procedure introduced above against an oblivious adversary, with perturbations $\sigma^k_t = \eta_t^k\sigma_1^k$ with $\eta_t^k = \sqrt{N_{t-1}^k}$ and $\sigma_1^k$ exponential with parameter $\eta$ is
\begin{align*}
\sum_{s=1}^t \ex_{\vlambda\sim\q_s}\ell_s(\vlambda) - \inf_{\vlambda\in\Lambda}\sum_{s=1}^t \ell_s(\vlambda)
\leq R_t = \sqrt{Kt}\left( \frac{D + 2CL}{\eta} + 2D\eta \right) \: .
\end{align*}
\end{theorem}

The expected regret of the FTPL algorithm in which the noises are independent in time and $\sigma_t^k$ is exponential with parameter $\eta/\eta_t^k$ is the same.

For non-oblivious adversaries, the quantity $\sum_{s=1}^t \ex_{\vlambda\sim\q_s}\ell_s(\vlambda) - \inf_{\vlambda\in\Lambda}\sum_{s=1}^t \ell_s(\vlambda)$ is also bounded by the same $R_t$, according to Lemma 4.1 of \citep{cesa2006prediction}.

\begin{proof}[Proof of Theorem~\ref{th:regret_FTPL}]
Regret decomposition:
for any $u$, the regret compared to $u$ is
\begin{align*}
\sum_{s=1}^{t} \ell_s(\vlambda_{s}(\vsigma_s, \hat{\vmu}_{s-1})) - \sum_{s=1}^{t} \ell_s(u)
&\leq \sum_{s=1}^{t} \ell_s(\vlambda_{s}(\vsigma_s, \hat{\vmu}_{s-1})) - \ell_s(\vlambda_{s+1}(\vsigma_s, \hat{\vmu}_{s-1})) \\
&\qquad + \sum_{s=1}^{t} \ell_s(\vlambda_{s+1}(\vsigma_s, \hat{\vmu}_{s-1})) - \sum_{s=1}^{t} \ell_s(u)
\end{align*}

\paragraph{Second term of the regret.} We are analysing here the regret of a noisy Be-The-Leader.
We first show by induction that
\begin{align*}
\sum_{s=1}^{t} &\ell_s(\vlambda_{s+1}(\vsigma_s, \hat{\vmu}_{s-1})) - \sum_{s=1}^{t} \ell_s(u)\\
&\leq \vsigma_t^\top d(\hat{\vmu}_{t-1}, u) + \sum_{s=1}^t\vsigma_s^\top d(\hat{\vmu}_{s-1}, \vlambda_{s+1}(\vsigma_s, \hat{\vmu}_{s-1})) - \vsigma_{s-1}^\top d(\hat{\vmu}_{s-2}, \lambda_{s+1}(\vsigma_s, \hat{\vmu}_{s-1}))
\end{align*}

Initialization: for all $u\in\Lambda$,
\begin{align*}
\ell_1(\vlambda_2(\vsigma_1, \hat{\vmu}_0))
&= \ell_1(\vlambda_2(\vsigma_1, \hat{\vmu}_0)) + \vsigma_1^\top d(\hat{\vmu}_0, \vlambda_2(\vsigma_1, \hat{\vmu}_0)) - \vsigma_1^\top d(\hat{\vmu}_0, \vlambda_2(\vsigma_1, \hat{\vmu}_0))\\
&\leq \ell_1(u) + \vsigma_1^\top d(\hat{\vmu}_0, u) - \vsigma_1^\top d(\hat{\vmu}_0, \vlambda_2(\vsigma_1, \hat{\vmu}_0)) \: .
\end{align*}

Let $A_1(u) = \vsigma_1^\top d(\hat{\vmu}_0, u) - \vsigma_1^\top d(\hat{\vmu}_0, \vlambda_2(\vsigma_1, \hat{\vmu}_0))$. Then for all $u\in\Lambda$, $\ell_1(\vlambda_2(\vsigma_1, \hat{\vmu}_0)) - \ell_1(u) \leq A_1(u)$.

Induction:
suppose that for all $u\in\Lambda$, $\sum_{s=1}^{t-1} \ell_s(\vlambda_{s+1}(\vsigma_s, \hat{\vmu}_{s-1})) \leq \sum_{s=1}^{t-1} \ell_s(u) + A_{t-1}(u)$, with
\begin{align*}
A_{t-1}(u) = \vsigma_{t-1}^\top d(\hat{\vmu}_{t-2}, u) + \sum_{s=1}^{t-1}\vsigma_{s-1}^\top d(\hat{\vmu}_{s-2}, \vlambda_{s+1}(\vsigma_s, \hat{\vmu}_{s-1})) - \vsigma_{s}^\top d(\hat{\vmu}_{s-1}, \vlambda_{s+1}(\vsigma_s, \hat{\vmu}_{s-1})) 
\end{align*}
where $\vsigma_0=0$. Apply it to $u = \vlambda_{t+1}(\vsigma_t, \hat{\vmu}_{t-1})$.
\begin{align*}
\sum_{s=1}^t \ell_s(\vlambda_{s+1}(\vsigma_s, \hat{\vmu}_{s-1}))
&\leq \sum_{s=1}^{t-1} \ell_s(\vlambda_{t+1}(\vsigma_t, \hat{\vmu}_{t-1})) + \ell_t(\vlambda_{t+1}(\vsigma_t, \hat{\vmu}_{t-1}))\\
&\qquad + A_{t-1}(\vlambda_{t+1}(\vsigma_t, \hat{\vmu}_{t-1}))\\
&= \sum_{s=1}^{t} \ell_s(\vlambda_{t+1}(\vsigma_t, \hat{\vmu}_{t-1})) + \vsigma_t^\top d(\hat{\vmu}_{t-1}, \vlambda_{t+1}(\vsigma_t, \hat{\vmu}_{t-1})) \\
&\qquad - \vsigma_t^\top d(\hat{\vmu}_{t-1}, \vlambda_{t+1}(\vsigma_t, \hat{\vmu}_{t-1})) + A_{t-1}(\vlambda_{t+1}(\vsigma_t, \hat{\vmu}_{t-1}))\\
&\leq \sum_{s=1}^{t} \ell_s(u) + \vsigma_t^\top d(\hat{\vmu}_{t-1}, u)\\
&\qquad - \vsigma_t^\top d(\hat{\vmu}_{t-1}, \vlambda_{t+1}(\vsigma_t, \hat{\vmu}_{t-1})) + A_{t-1}(\vlambda_{t+1}(\vsigma_t, \hat{\vmu}_{t-1})) \: .
\end{align*}
We obtain
\begin{align*}
A_t(u) -  \vsigma_t^\top d(\hat{\vmu}_{t-1}, u)
&= A_{t-1}(\vlambda_{t+1}(\vsigma_t, \hat{\vmu}_{t-1})) - \vsigma_t^\top d(\hat{\vmu}_{t-1}, \vlambda_{t+1}(\vsigma_t, \hat{\vmu}_{t-1})) \\
&= \sum_{s=1}^t\vsigma_{s-1}^\top d(\hat{\vmu}_{s-2}, \vlambda_{s+1}(\vsigma_s, \hat{\vmu}_{s-1})) - \vsigma_{s}^\top d(\hat{\vmu}_{s-1}, \vlambda_{s+1}(\vsigma_s, \hat{\vmu}_{s-1}))\: .
\end{align*}
End of the induction proof.

We now bound $A_t(u)$. First we write
\begin{align*}
A_t(u) -  \vsigma_t^\top d(\hat{\vmu}_{t-1}, u)
&= \sum_{s=1}^t\vsigma_{s-1}^\top d(\hat{\vmu}_{s-2}, \vlambda_{s+1}(\vsigma_s, \hat{\vmu}_{s-1})) - \vsigma_{s}^\top d(\hat{\vmu}_{s-1}, \vlambda_{s+1}(\vsigma_s, \hat{\vmu}_{s-1}))\\
&= \sum_{s=1}^t\vsigma_s^\top [d(\hat{\vmu}_{s-2}, \vlambda_{s+1}(\vsigma_s, \hat{\vmu}_{s-1})) - d(\hat{\vmu}_{s-1}, \vlambda_{s+1}(\vsigma_s, \hat{\vmu}_{s-1}))]\\
&\quad + \sum_{s=1}^t(\vsigma_{s-1} - \vsigma_{s})^\top d(\hat{\vmu}_{s-2}, \vlambda_{s+1}(\vsigma_s, \hat{\vmu}_{s-1})) \: .
\end{align*}
We now bound separately the two sums. The first one uses the Lipschitz-continuity of $d$ and the fact that successive $\hat{\vmu}_t$ are not far from each other.
\begin{align*}
\ex\sum_{s=1}^t & \vsigma_s^\top [d(\hat{\vmu}_{s-2}, \vlambda_{s+1}(\vsigma_s, \hat{\vmu}_{s-1})) -  d(\hat{\vmu}_{s-1}, \vlambda_{s+1}(\vsigma_s, \hat{\vmu}_{s-1}))]\\
&=    \ex\sum_{s=1}^t \vsigma_s^{k_{s-1}} [d(\hat{\mu}_{s-2}^{k_{s-1}}, \lambda_{s+1}^{k_{s-1}}(\vsigma_s, \hat{\vmu}_{s-1})) -  d(\hat{\mu}_{s-1}^{k_{s-1}}, \lambda_{s+1}^{k_{s-1}}(\vsigma_s, \hat{\vmu}_{s-1}))]\\
&\leq \sum_{s=1}^t \ex [ \sigma_s^{k_{s-1}} ] L |\hat{\mu}_{s-1}^{k_{s-1}} - \hat{\mu}_{s-2}^{k_{s-1}}| 
\leq CL  \sum_{s=1}^t \ex [ \sigma_s^{k_{s-1}} ]  \frac{1}{N_{s-1}^{k_{s-1}}} 
\leq \frac{CL}{\eta} \sum_{s=1}^t \frac{\eta_s^{k_{s-1}}}{N_{s-1}^{k_{s-1}}} \: .
\end{align*}
For $\eta_t^k$ non-decreasing in $t$, $\vsigma_{s-1} - \vsigma_s$ has non-positive coordinates and the second sum is negative.

We obtain
\begin{align*}
\ex A_t(u)
&\leq \ex \vsigma_t^\top d(\hat{\vmu}_{t-1}, u) + \frac{CL}{\eta} \sum_{s=1}^t \frac{\eta_s^{k_{s-1}}}{N_{s-1}^{k_{s-1}}}
\leq   \frac{D \Vert\eta_t\Vert_1}{\eta} + \frac{CL}{\eta} \sum_{s=1}^t \frac{\eta_s^{k_{s-1}}}{N_{s-1}^{k_{s-1}}} \ .
\end{align*}

\paragraph{First term of the regret.}

Remark that
$\vlambda_{t+1}(\vsigma, \hat{\vmu}_{t-1}) = \vlambda_t(\vsigma+e_{k_t}, \hat{\vmu}_{t-1})$ . Let $f$ be the density of the distribution of $\sigma_t$.
In expectation, the first term of the regret is
\begin{align*}
\ex_{\vsigma_t} &[\ell_t(\vlambda_t(\vsigma_t, \hat{\vmu}_{t-1})) - \ell_t(\vlambda_{t+1}(\vsigma_t, \hat{\vmu}_{t-1}))]\\
&= \int_{\vsigma_t} [\ell_t(\vlambda_t(\vsigma_t, \hat{\vmu}_{t-1})) - \ell_t(\vlambda_t(\vsigma_t+e_{k_t}, \hat{\vmu}_{t-1}))] f(\vsigma_t) d\vsigma_t\\
&= \int_{\vsigma_t} \ell_t( \vlambda_t(\vsigma_t, \hat{\vmu}_{t-1})) (f(\vsigma_t)-f(\vsigma_t - e_{k_t})) d\vsigma_t\\
\end{align*}
By positivity of $\ell_t$ (since it is a divergence),
\begin{align*}
\ex_{\vsigma_t} &[\ell_t(\vlambda_t(\vsigma_t, \hat{\vmu}_{t-1})) - \ell_t(\vlambda_{t+1}(\vsigma_t, \hat{\vmu}_{t-1}))]\\
&\leq \int_{\vsigma_t} \ell_t(\vlambda_t(\vsigma_t, \hat{\vmu}_{t-1})) \mathbb{I}\{f(\vsigma_t)-f(\vsigma_t - e_{k_t}) > 0\} (f(\vsigma_t)-f(\vsigma_t - e_{k_t})) d\vsigma_t\\
&\leq D \int_{\vsigma_t}\mathbb{I}\{f(\vsigma_t)-f(\vsigma_t - e_{k_t}) > 0\} (f(\vsigma_t)-f(\vsigma_t - e_{k_t})) d\vsigma_t\\
&\leq D \int_{\vsigma_t}\mathbb{I}\{f(\vsigma_t)-f(\vsigma_t - e_{k_t}) > 0\} f(\vsigma_t) d\vsigma_t\\
&= D \int_{\sigma_t^{k_t}\leq 1} f(\vsigma_t) d\vsigma_t\\
&= D (1 - e^{-\eta/\eta_t^{k_t}}) \\
&\leq D \eta/\eta_t^{k_t}  \: .
\end{align*}

\paragraph{Putting things together.}

Choose $\eta_t^k = \sqrt{N_{t-1}^k}$.
\begin{align*}
\ex R_t
&\leq  D  \frac{\Vert\eta_t\Vert_1}{\eta} + \frac{CL}{\eta} \sum_{s=1}^t \frac{\eta_s^{k_{s-1}}}{N_{s-1}^{k_{s-1}}} + D\eta \sum_{s=1}^t \frac{1}{\eta_s^{k_s}}  
\leq \sqrt{Kt}\left( \frac{D + 2CL}{\eta} + 2D\eta \right) \: . 
\end{align*}

\end{proof}

\paragraph{Approximation of $\q_t$ by an empirical distribution.} We want the $k$-player to use optimistic best-response to $\q_t$. This requires the computation of
\begin{align*}
\argmax_{k\in[K]} U_t^k && \text{with } U_t^k = \max_{\xi\in\{a_t^k, b_t^k\}}\ex_{\vlambda \sim \q_t} d(\xi, \lambda^k) \: .
\end{align*}
for some values $a_t^k, b_t^k$.

Since we cannot compute an expectation under $\q_t$ exactly, we compute instead the expectation under an empirical distribution based on $t$ samples $\vlambda_t^{(1)},\ldots,\vlambda_t^{(t)}$ of $\q_t$. For all $\xi$, $d(\xi, \lambda^k)$ is bounded by $D$. Hence, by Hoeffding's inequality,
\begin{align*}
\pr\left\{ \frac{1}{t}\sum_{j=1}^t d(\xi, \lambda_t^{(j)k}) - \ex_{\vlambda \sim \q_t} d(\xi, \lambda^k) \geq \sqrt{\frac{3D^2\log(t)}{2t}} \right\}
&\leq \frac{1}{t^3} \: .
\end{align*}
In the concentration analysis of the algorithm, we replace $\mathcal{E}_t$ by $\mathcal{E}_t \cap \mathcal{E}_t'$ with
\begin{align*}
\mathcal{E}_t' = \left\{\forall k\in[K], \forall s\leq t, \forall \xi\in\{a_s^k, b_s^k\}\: \frac{1}{s}\sum_{j=1}^s d(\xi, \lambda_s^{(j)k}) - \ex_{\vlambda \sim \q_s} d(\xi, \lambda^k) \leq D\sqrt{\frac{3\log(t)}{2t}} \right\}
\end{align*}
It verifies $\sum_{t=1}^{+\infty} \pr({\mathcal{E}_t'}^c) \leq 2K\sum_{t=1}^{+\infty} 1/t^2 \leq K\pi^2/3$ .

Under the event $\mathcal{E}_t'$,
\begin{align*}
\sum_{s=1}^t \frac{1}{s}\sum_{j=1}^s d(\xi, \lambda_s^{(j)k}) - \sum_{s=1}^t \ex_{\vlambda \sim \q_s} d(\xi, \lambda^k) \leq D\sqrt{\frac{3}{2}t\log(t)} \: .
\end{align*}
We obtain that the procedure based on these empirical distributions has $\mathcal{O}(\sqrt{t\log t})$ regret.

\section{On the statistical assumptions}\label{sec:ass_discussion}
\subsection{The sub-Gaussian assumption}\label{sec:ass_discussion_subG}

The natural coordinate-wise concentration events for exponential families have the form $N_t^k d(\hat{\mu}_t^k, \mu^k) \leq c$ for some constant $c>0$.
In our proofs, we need then to relate $d(\hat{\mu}_t^k, \lambda^k)$ and $d(\mu^k, \lambda^k)$ for a given $\lambda^k$ under such a concentration constraint. However, we now show that for some convex function $\phi$ (such that $d$ is the associated Bregman divergence), these two quantities could be very far apart even under the constraint $d(\hat{\mu}_t^k, \mu^k)=0$.

If $d(\hat{\mu}_t^k, \mu^k)=0$, we have the equalities
\begin{align*}
d(\hat{\mu}_t^k, \lambda^k) - d(\mu^k, \lambda^k)
&= d(\hat{\mu}_t^k, \mu^k) + (\hat{\mu}_t^k - \mu^k)(\phi'(\mu^k) - \phi'(\lambda^k)) \\
&= (\hat{\mu}_t^k - \mu^k)(\phi'(\mu^k) - \phi'(\lambda^k)) \: .
\end{align*}

Let $\phi:\R\to\R$ be defined by $\phi(x)=\max\{0,x\}$. Let $\lambda^k = 1$, $\mu^k = -1$ and $\hat{\mu}_t^k < -1$. Then
\begin{align*}
d(\hat{\mu}_t^k, \mu^k) &= 0 \: ,\\
d(\hat{\mu}_t^k, \lambda^k) - d(\mu^k, \lambda^k) &= |\hat{\mu}_t^k - \mu^k| \: .
\end{align*} 
In that example, the constraint on $d(\hat{\mu}_t^k, \mu^k)$ is not sufficient to bound $d(\hat{\mu}_t^k, \lambda^k) - d(\mu^k, \lambda^k)$.

The example exploits the piecewise linearity of $\phi$. Such a function $\phi$ cannot arise from an exponential family. Indeed, for an exponential family $\phi$ is the convex conjugate of a cumulant generating function. In particular, $\phi$ is strictly convex. But it could still have very low curvature (for example for an exponential distribution with high mean). The sub-Gaussian assumption ensures that $\phi$ is strongly convex.

Our work and previous parametric pure exploration papers treat $d$ as a general Bregman divergence. The present example shows that either we need to also use more specific properties of $d$ due to the fact that it is a Kullback-Leibler divergence, or we need to impose additional assumptions like sub-Gaussianity.

\subsection{The upper bound assumption}\label{sec:ass_discussion_lipschitz}

A first way to relax the assumption that $\mathcal{M} \subseteq [\mu_{\min}, \mu_{\max}]^K$ is to remark that we do not need to bound $d(\mu,\lambda)$ for any $\mu$ and $\lambda$.

For $\vmu \in \mathcal{M}$ and $\w\in\triangle_K$, let $\vlambda(\vmu, \w) = \argmin_{\vlambda \in \neg i} \sum_{k=1}^K w^k d(\mu^k, \lambda^k)$.
Our proofs are valid for example under the following assumption.
\begin{assumption}
There exists $D>0$ and $L>0$ such that for all $\w\in\triangle_K$, for all $\vmu\in\mathcal{M}$, $\Vert d(\vmu, \vlambda(\vmu, \w)) \Vert_\infty \leq D$ and $\Vert \phi'(\vmu) - \phi'(\vlambda(\vmu, \w))\Vert_\infty \leq L$.
\end{assumption}

We could also use the concentration events to replace it with weaker hypotheses.
Under event $\mathcal{E}_t$ and with Assumption~\ref{ass:sub-Gaussian}, for all $s\leq t$, $\Vert d(\hat{\vmu}_s, \vmu) \Vert_\infty \leq f(t)$ and $\Vert\hat{\vmu}_s - \vmu\Vert_\infty \leq \sqrt{2\sigma^2 f(t)}$. That is, we get from concentration only, without assumptions, that $\hat{\vmu}_t$ is in a bounded set around $\vmu$. We can then quantify $L$ and $D$ on that set.

Let $L_\vmu = \sup_{\w\in\triangle_k} \max_k |\phi'(\mu^k) - \phi'(\lambda(\vmu,\w)^k) |$.

\begin{assumption}
For all $\vmu\in\mathcal{M}$, $L_\vmu$ is finite.
\end{assumption}
This is true for BAI, where $L_\vmu \leq \phi'(\max_k\mu^k) - \phi'(\min_k\mu^k)$.

\begin{assumption}
There exists $M>0$ such that $\vmu\mapsto L_\vmu$ is $M$-Lipschitz for the $\ell^\infty$ norm.
\end{assumption}
This is true for BAI on sets on which $\phi'$ is Lipschitz. For example, it is true on $\R$ for Gaussian arm distributions, but is still only true in intervals of the form $[\varepsilon, 1-\varepsilon]$ for Bernoulli distributions.

Then for any $\vmu,\vxi$ and $\vlambda_\vmu$ minimal point for $\vmu$, for any coordinate $k \in [K]$ (omitted in the computations),
\begin{align*}
d(\mu, \lambda_\mu)
&= d(\xi, \lambda_\mu) + (\mu - \xi)(\phi'(\mu) - \phi'(\lambda_\mu)) - d(\xi,\mu)\\
&\geq d(\xi, \lambda_\mu) -|\mu-\xi| L_\mu - d(\xi,\mu)\\
&\geq d(\xi, \lambda_\mu) -|\mu-\xi| L_\xi - M(\mu - \xi)^2 - d(\xi,\mu)\\
&\geq d(\xi, \lambda_\mu) - L_\xi\sqrt{2\sigma^2 \min\{d(\mu,\xi),d(\xi,\mu)\}} - 2\sigma^2 M \min\{d(\mu,\xi),d(\xi,\mu)\} - d(\xi,\mu)
\end{align*}
Examples for the quantities used in the proofs:
\begin{align*}
d(\mu, \lambda_\mu)
&\geq d(\hat{\mu}_{s-1}, \lambda_\mu) - L_\mu \sqrt{2\sigma^2 d(\hat{\mu}_{s-1},\mu)} - d(\hat{\mu}_{s-1},\mu)
\\
d(\hat{\mu}_t, \lambda_{\hat{\mu}_t})
&\geq d(\mu, \lambda_{\hat{\mu}_t}) - L_\mu \sqrt{2\sigma^2 d(\hat{\mu}_t,\mu)} - 2\sigma^2 M d(\hat{\mu}_t,\mu) - d(\mu,\hat{\mu}_t)
\end{align*}
The proofs must then be adapted to account for the additional terms in these inequalities.

\section{Numerical Experiments}\label{appx:experiments}
\subsection{Best Arm}

\begin{figure}[htp]
  \centering
  \subfigure[Bernoulli bandit $\vmu = (0.5, 0.45, 0.43, 0.4)$, $\w^* = (0.42, 0.39, 0.14, 0.06)$]{
    \includegraphics[width=.45\textwidth]{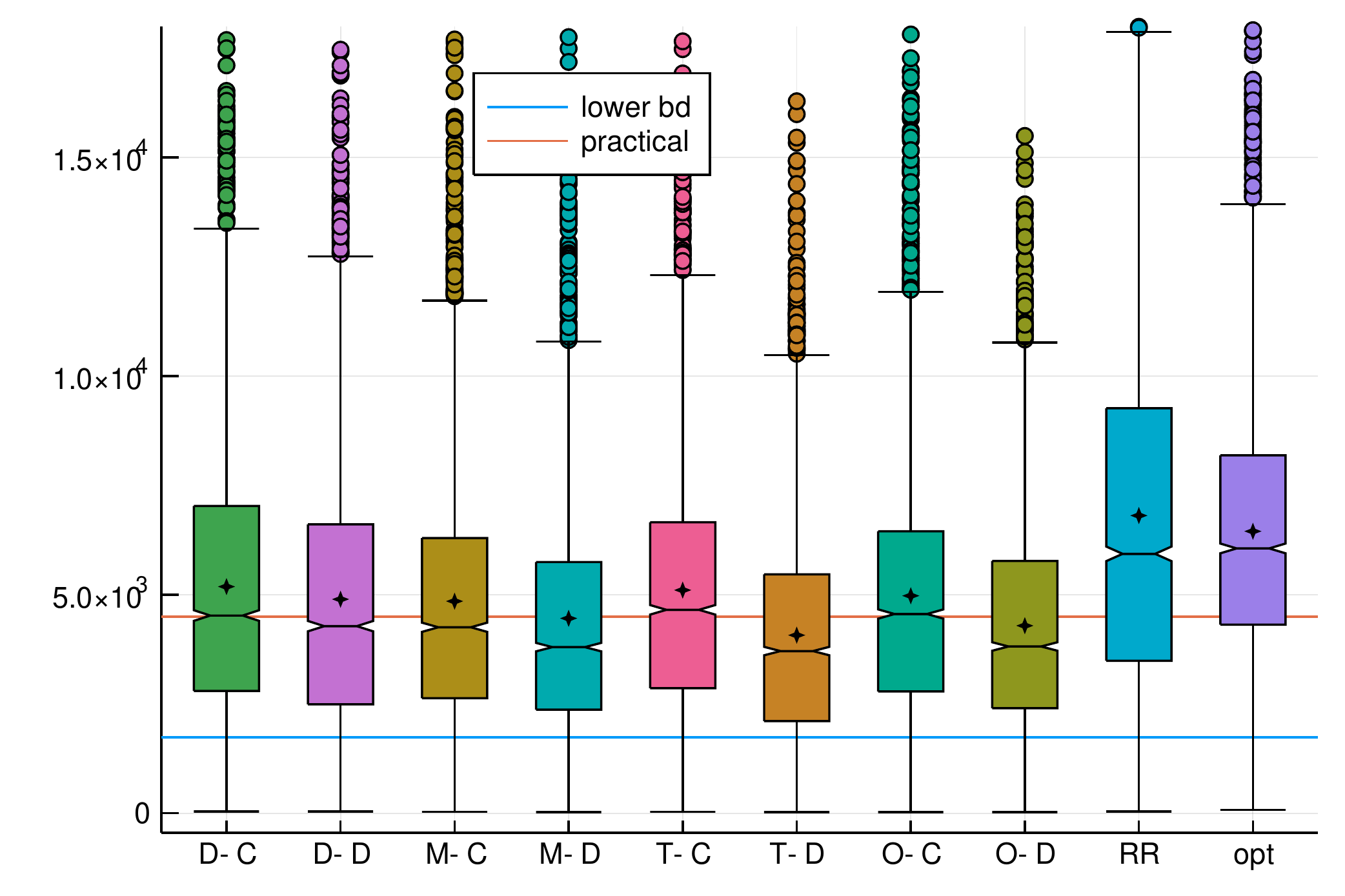}
  }%
  \quad
  \subfigure[Bernoulli bandit $\vmu = (0.3, 0.21, 0.2, 0.19, 0.18)$, $\w^* = (0.34, 0.25, 0.18, 0.13, 0.10)$]{
    \includegraphics[width=.45\textwidth]{figs/experiment_bai2_1}
  }
  \caption{Best Arm experiments from \cite{garivier2016optimal}. In both cases $\delta = 0.1$. Plots show $3000$ runs.}
\end{figure}

\begin{figure}[htp]
  \centering
  \subfigure[$\delta=0.1$]{
    \includegraphics[width=.45\textwidth]{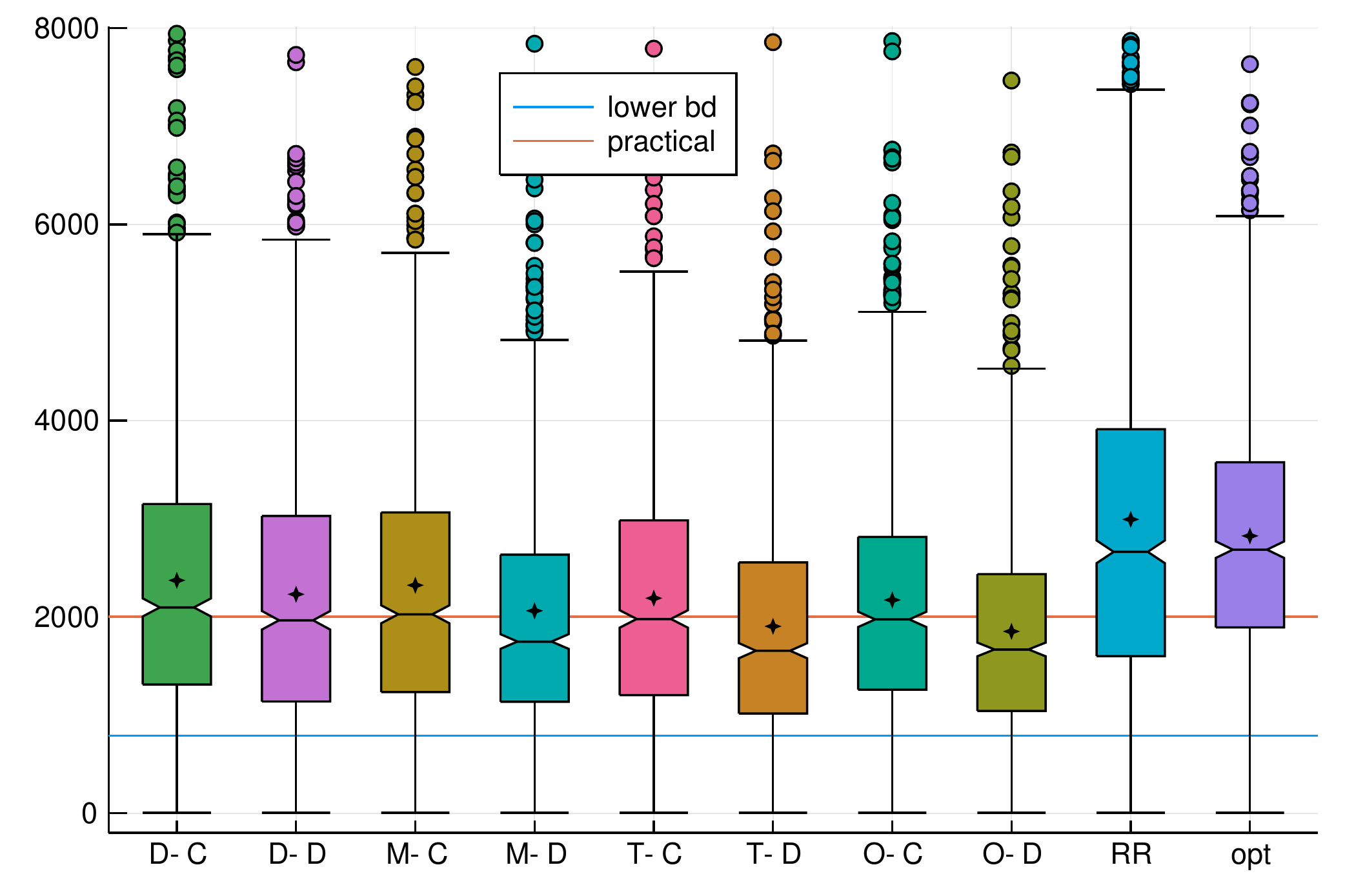}%
  }%
  \quad
  \subfigure[$\delta=0.01$]{
    \includegraphics[width=.45\textwidth]{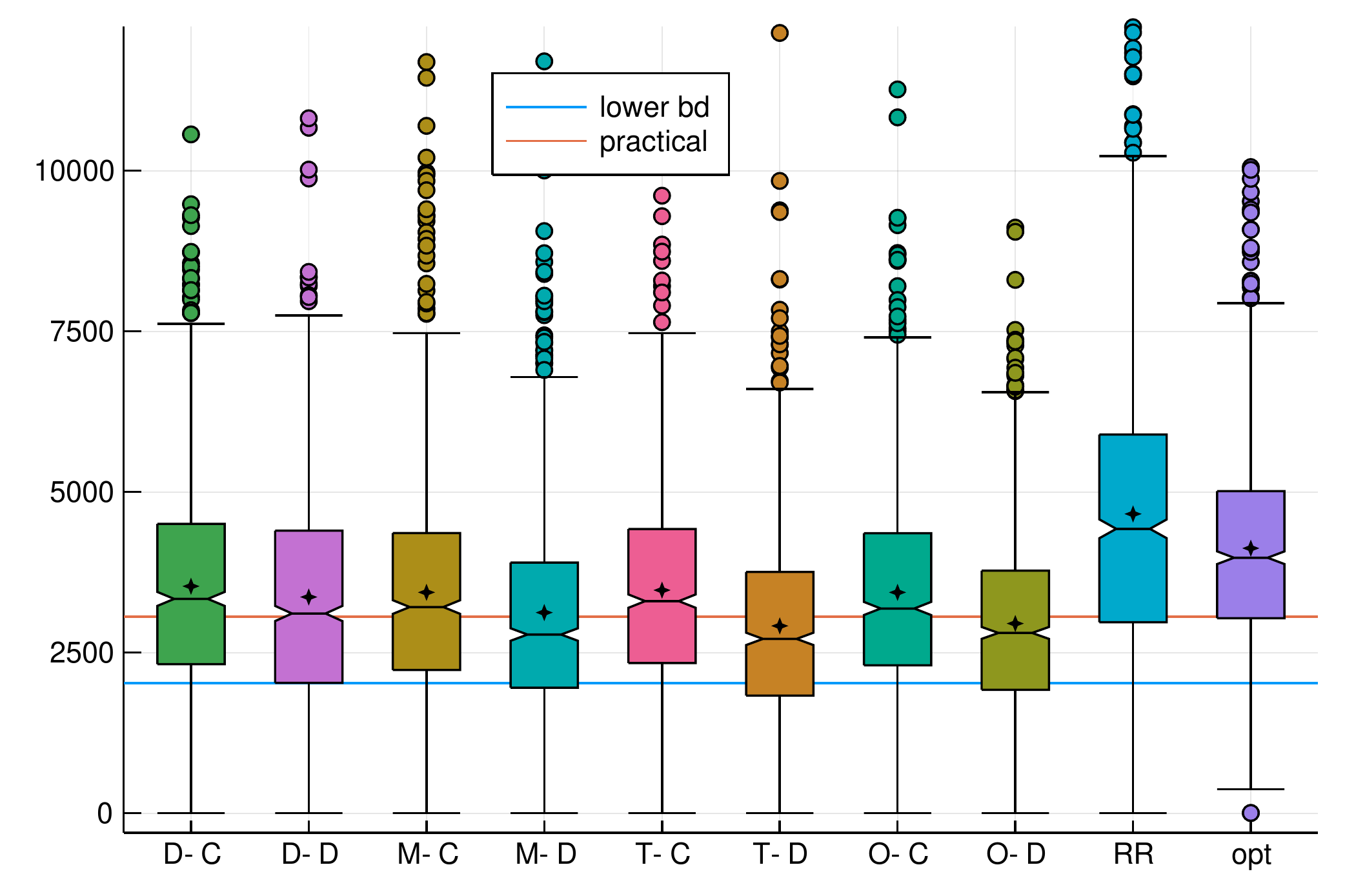}
  }
  \caption{Best Arm experiment from \cite{Menard19}. Gaussian bandit $\vmu = (1., 0.85, 0.8, 0.7)$, $\w^* = (0.41, 0.38, 0.15, 0.06)$. Plots show $3000$ runs.}
\end{figure}

\clearpage
\subsection{Minimum Threshold}\label{appx:exper.mt}

\begin{figure}[htp]
  \centering
  \subfigure[Gaussian bandit $\vmu = (-1, \ldots, 1)$ with $K=10$ arms and $\delta = e^{-23}$, $\w^* = \e_1$]{
    \includegraphics[width=.45\textwidth]{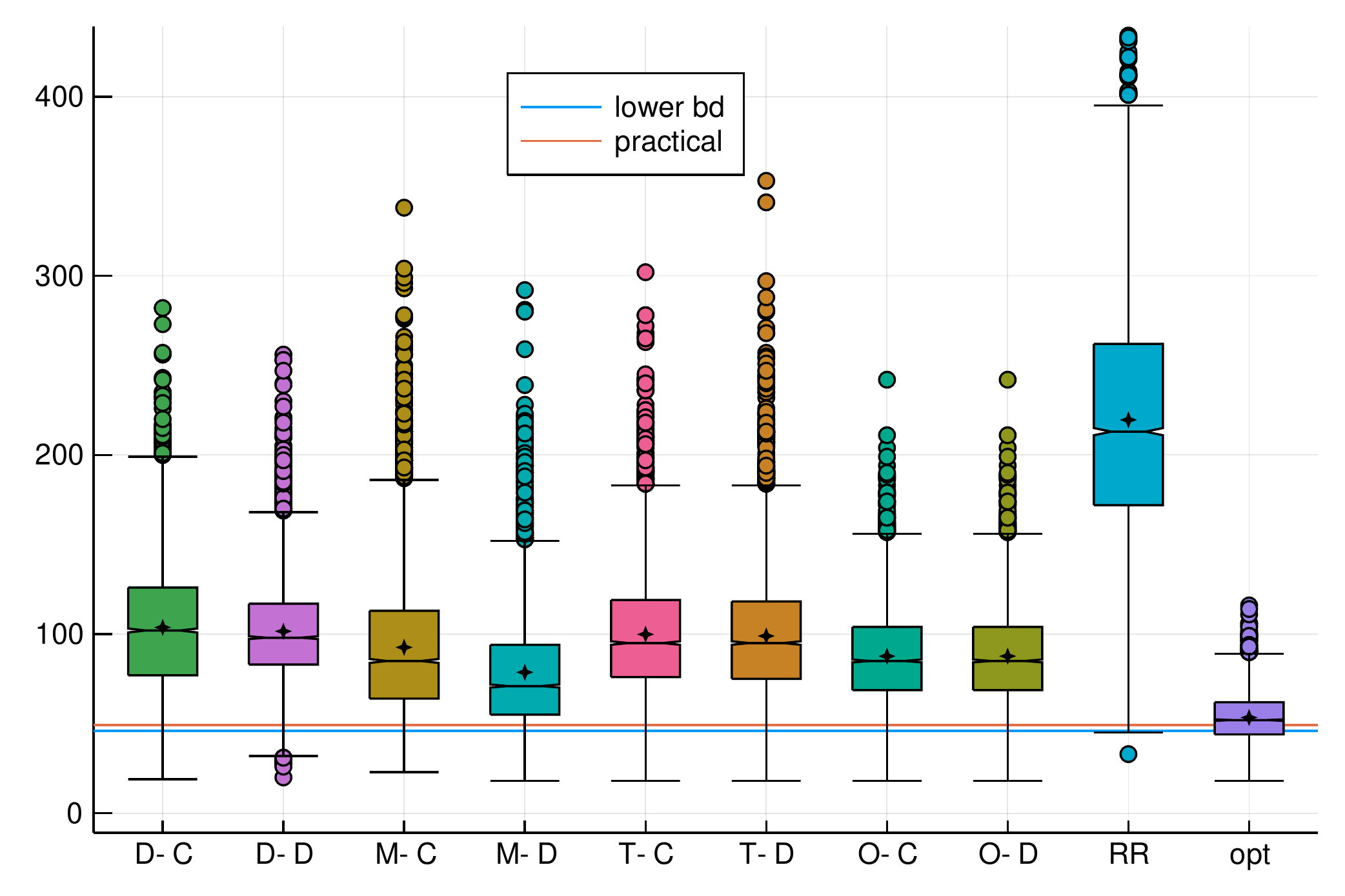}
  }%
  \quad
  \subfigure[Gaussian bandit $\vmu = (0.5, \ldots, 1)$ with $K=5$ arms and $\delta = e^{-7}$, $\w^*=(0.38, 0.24, 0.17, 0.12, 0.09)$]{
    \includegraphics[width=.45\textwidth]{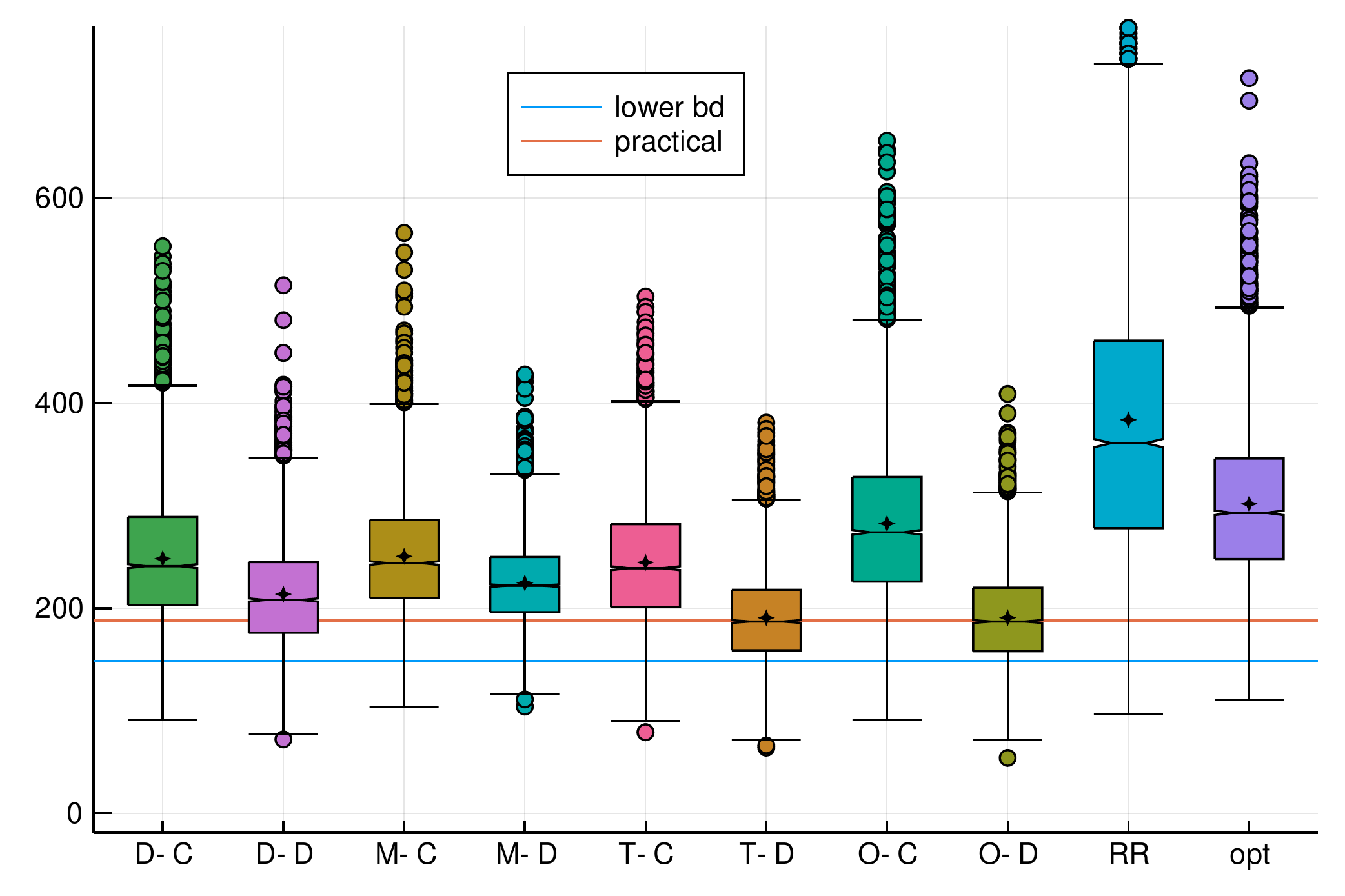}
  }
  \caption{Minimum Threshold experiments from \cite{kaufmann2018sequential} with threshold $\gamma=0$. Plots show $5000$ runs.}
\end{figure}

\begin{figure}[htp]
  \centering
  \subfigure[$\delta =0.1$]{
    \includegraphics[width=.45\textwidth]{figs/experiment_threshold3_1}%
  }%
  \quad
  \subfigure[$\delta =0.0001$]{
    \includegraphics[width=.45\textwidth]{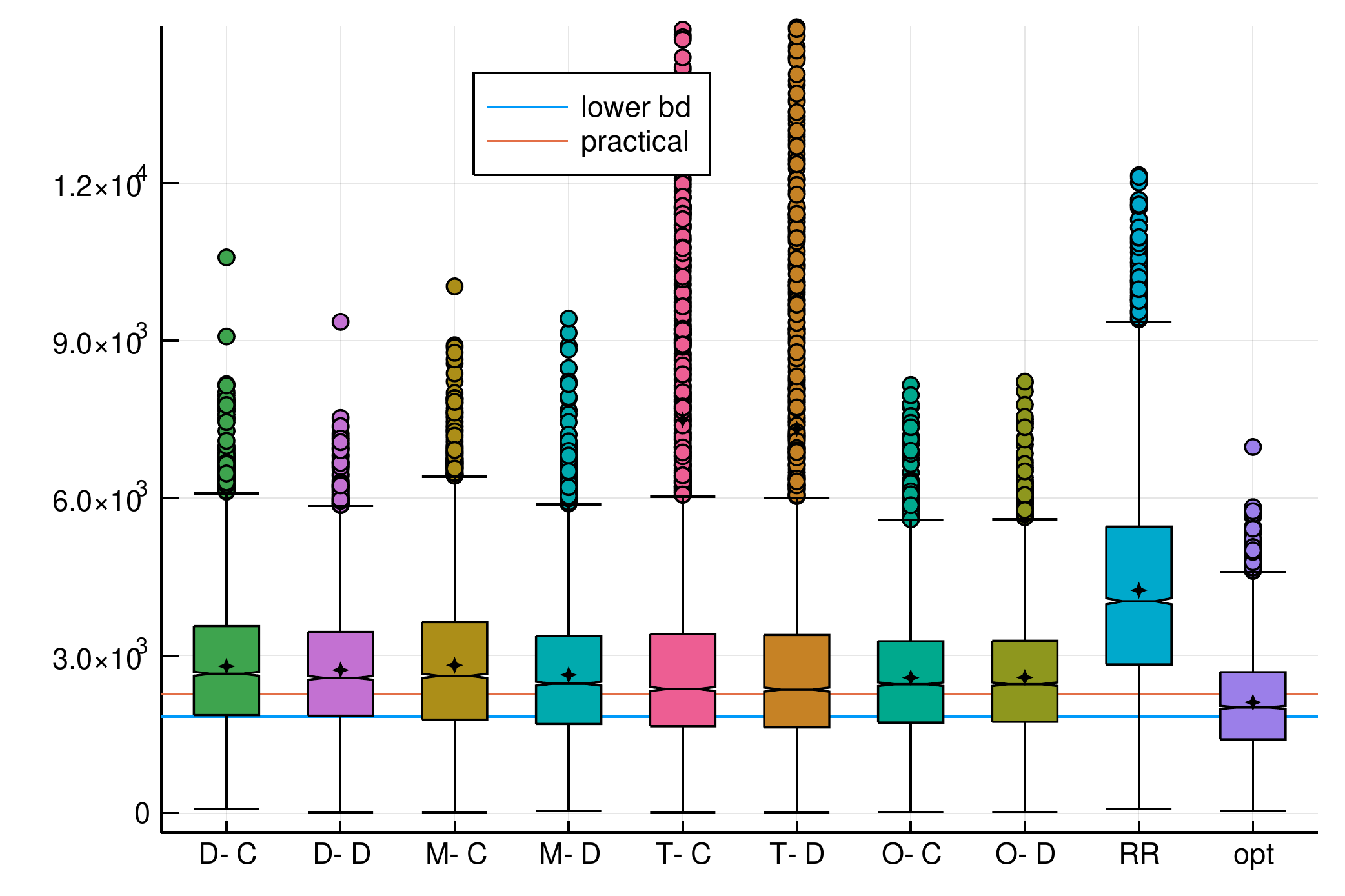}
  }
  \quad
    \subfigure[$\delta = 10^{-10}$]{
    \includegraphics[width=.45\textwidth]{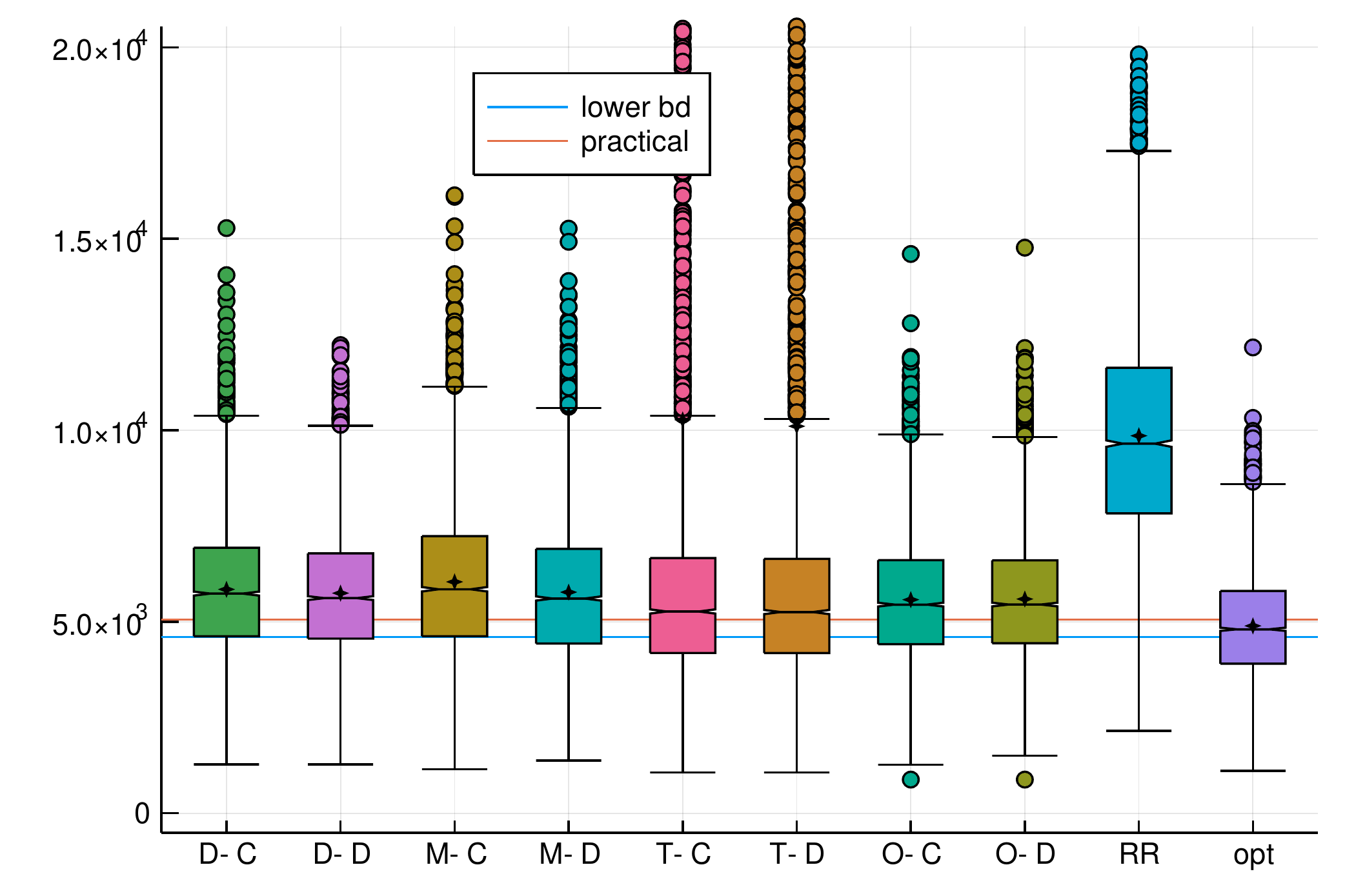}
  }
  \quad
    \subfigure[$\delta = 10^{-20}$\label{fig:tas.bad}]{
    \includegraphics[width=.45\textwidth]{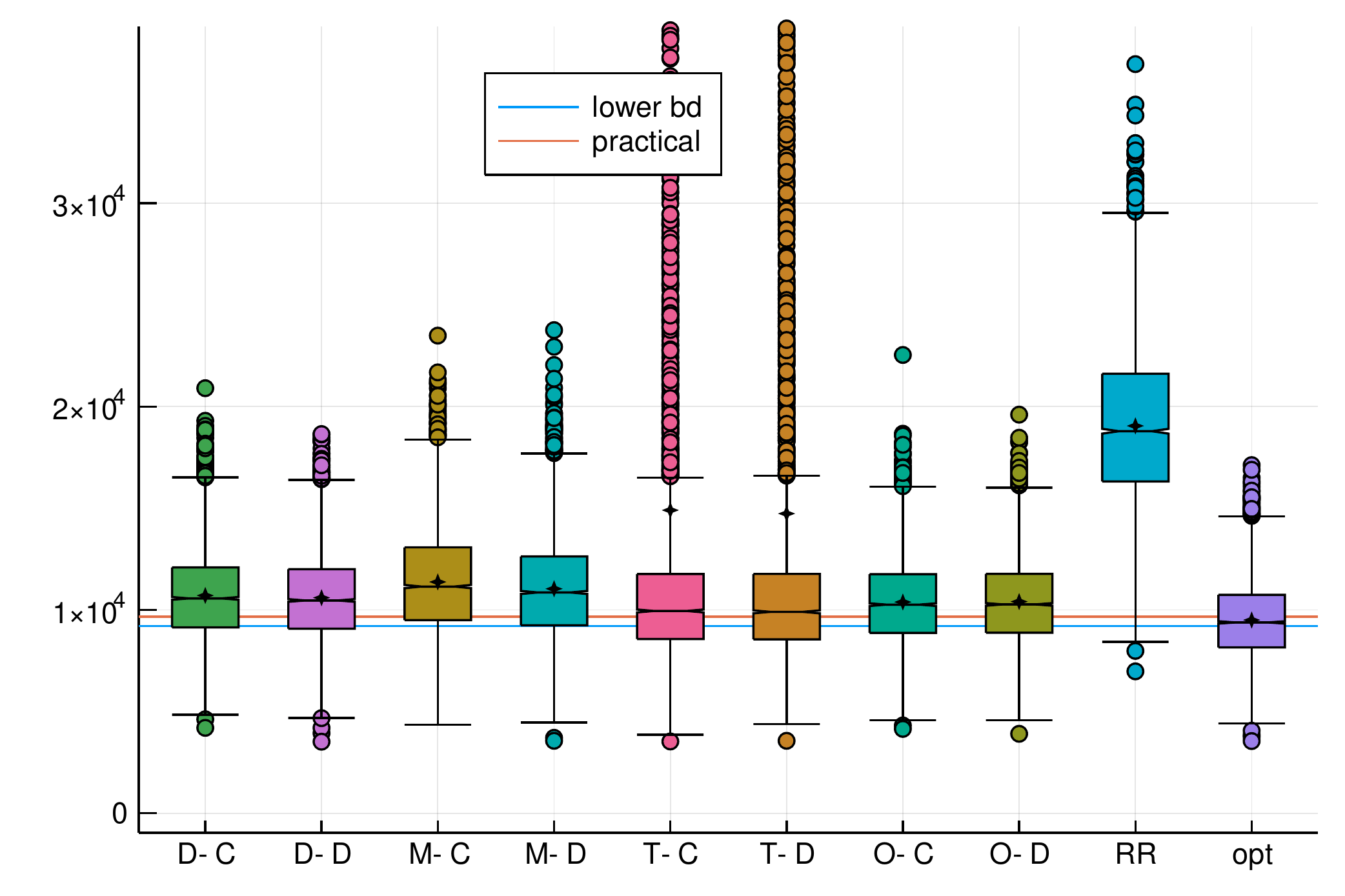}
  }
  \caption{Minimum Threshold experiment (new): Gaussian bandit $\vmu = (0.5, 0.6)$ with threshold $\gamma=0.6$, $\w^* = \e_1$. Note the excessive sample complexity of Track-and-Stop (T-C and T-D). Plots show $5000$ runs.}\label{fig:tas.bad}
\end{figure}

The reason for the bad performance of Track-and-Stop in Figure~\ref{fig:tas.bad} is that with small but non-negligible probability the algorithm finds $\hat \mu_t^1 \gg \gamma$ estimated too high at some early $t$. In this situation $\w^*(\vmu_t)$ will be $\e_2$ (exactly if $\hat \mu_t^2 \le \gamma$, approximately if $\hat \mu_t^2 > \gamma$), and constantly pulling arm $2$ will not correct the estimate of arm $1$. \textbf{T} relies on forced exploration to correct the estimate.

\end{document}